%% file: main.tex
\documentclass[twoside,11pt]{article}

\usepackage{blindtext}

\usepackage[preprint]{jmlr2e}
\usepackage{epsfig} 
\usepackage{times} 
\usepackage{amsmath} 
\usepackage{amssymb,amsbsy}  
\usepackage{amsfonts}
\usepackage{mathptmx} 
\usepackage{mathtools}
\usepackage{tabularx}

\usepackage{booktabs}
\usepackage{arydshln}
\usepackage{epstopdf}
\usepackage{textcomp}
\usepackage{float}
\usepackage{epic,color}
\usepackage{mathrsfs}
\usepackage{dsfont,MnSymbol}
\usepackage{algorithm}
\usepackage{algorithmic}
\usepackage{multirow} 
\usepackage{multicol}
\usepackage{soul}
\usepackage[T1]{fontenc}

\usepackage{tikz}
\usetikzlibrary{arrows, positioning, calc, arrows.meta, fit}

\newif\ifshownewcontent
\shownewcontenttrue  
\ifshownewcontent
  
\else
  
\fi

\input{_definitions}

\input{jin_thesis_macros}

\usepackage{lastpage}
\jmlrheading{23}{2025}{1-\pageref{LastPage}}{X/XX; Revised X/XX}{X/XX}{25-0000}{H.-S. Chang and P. G. Mehta}

\ShortHeadings{Dual Filter}{Chang and Mehta}
\firstpageno{1}

\begin{document}

\title{Dual Filter: A Transformer-like Inference \\ Architecture for
  Hidden Markov Models}

\author{
  \name Heng-Sheng Chang \email hschang@illinois.edu \\
  \addr Coordinated Science Laboratory\\
  University of Illinois Urbana-Champaign\\
  Urbana, IL 61801, USA
  \AND
  \name Prashant G. Mehta \email mehtapg@illinois.edu \\
  \addr Coordinated Science Laboratory\\
  University of Illinois Urbana-Champaign\\
  Urbana, IL 61801, USA
}

\editor{My editor}

\maketitle

\begin{abstract}
This paper presents a mathematical framework for causal nonlinear prediction in settings where observations are generated from an underlying hidden Markov model (HMM).  
Both the problem formulation and the proposed solution are motivated by the decoder-only transformer architecture, in which a finite sequence of observations (tokens) is mapped to the conditional probability of the next token. 
Our objective is not to construct a mathematical model of a transformer. 
Rather, our interest lies in deriving, from first principles, transformer-like
architectures that solve the prediction problem for which the transformer is designed.  
The proposed framework is based on an original optimal control approach, where the prediction objective (MMSE) is reformulated as an optimal control problem.
An analysis of the optimal control problem is presented leading to a fixed-point equation on the space of probability measures. 
To solve the fixed-point equation, we introduce the dual filter, an iterative algorithm that closely parallels the architecture of decoder-only transformers.  
These parallels are discussed in detail along with the relationship to prior work on mathematical modeling of transformers as transport on the space of probability measures. 
Numerical experiments are provided to illustrate the performance of the algorithm using parameter values typical of research-scale transformer models. 
\end{abstract}

\begin{keywords}
  Nonlinear predictor, transformer, stochastic control, optimal control, hidden
  Markov model.
\end{keywords}

\section{Introduction}
\label{sec:intro}

Let $\bO=\{0,1,2,\hdots,m\}$ denote a finite set called the vocabulary.
An element of $\bO$ is referred to as a token.
A sequence of $T$ tokens is an $\bO^T$-valued random vector, denoted by
$\{Z_1,Z_2,\hdots,Z_T\}$.
A decoder-only transformer is an algorithm to compute the conditional probability of the next token (see~\cite{phuong2022formal}):
\begin{equation*}
  \sP(Z_{T+1}=z \mid Z_1,Z_{2},\hdots,Z_T),\quad z\in \bO.  
\end{equation*}
During inference with a well-trained transformer, the conditional probability is often sparse---that is, only a small subset of tokens has non-negligible probability. This sparsity is useful  for efficient sampling in generative AI applications~\citep{wolfram2023chatgpt}.

There are two distinguishing features of the decoder-only transformer architecture:
\begin{enumerate}
  \item Even though only the conditional probability at the terminal time $t=T$ is of interest, conditional probabilities are also computed for intermediate times,
    \begin{equation*}
      \sP(Z_{t+1}=z \mid Z_1,Z_2,\hdots,Z_t),\quad z\in \bO,\quad t=1,2,\hdots,T.
    \end{equation*}

  \item In all cases, the conditional probability of the next token is expressed as a causal, nonlinear function of the past tokens, implemented through a procedure known as causal masking. 
  In this paper, we refer to such a function as a nonlinear predictor (a formal definition is given after the model has been introduced).  
\end{enumerate}
The second item is in contrast to a
recurrent neural network (RNN) architecture, where a hidden state is
stored and recursively updated~\citep{graves2013generating,dai2019transformer}.  

A transformer architecture is reminiscent of the classical Wiener filter.
Recall that a Wiener filter computes the conditional expectation of a Gaussian process, also denoted (with slight abuse of notation) as $[Z_1,Z_2,\hdots,Z_T]$, in the following
causal form:
\begin{equation*}
  \E(Z_{T+1} \mid Z_1,Z_{2},\hdots,Z_T) = \text{(constant)} + \sum_{t=1}^T u_t^\transpose Z_t.  
\end{equation*}
The right-hand side is an example of a linear predictor where $u:=\{u_t\in\Re^{m\times m}:1\leq t\leq T\}$ are deterministic weights, to be designed or learned.  
The Wiener filtering theory is concerned with the synthesis of the optimal weights that
yield the conditional expectation~\cite[Ch.~7]{kailath2000linear}.

The objective of this paper is to develop both theory and algorithms for a nonlinear predictor, that computes the conditional probability of the next token, for a large but finite vocabulary. 
Our focus is exclusively on inference, not on learning. 
To this end, we adopt a model-based approach based on a hidden Markov model (HMM), where the observed tokens are generated from an underlying hidden stochastic process. 
This process evolves as a Markov chain, taking values in a finite state space of dimension $d$. 
We begin by describing the model and then introducing the problem.

\subsection{Math Preliminaries: Hidden Markov model and the nonlinear filter}

Throughout this paper, we consider discrete-time stochastic processes on a finite time-horizon $\mathbb{T}=\{0,1,2,\hdots,T\}$ with $T<\infty$. 
Fix the probability space $(\Omega,\clF_T,\sP)$ along with the filtration $\{\clF_t:t\in\mathbb{T}\}$ with respect to which all the stochastic processes are adapted.
A hidden Markov model (HMM) is specified by a pair of stochastic processes $(X,Z)$ defined as follows (see~\cite{elliott2008hidden,Moulines2006inference}): 
\begin{enumerate}
  \item The {state-space} $\bS=\{1,2,\hdots,d\}$ is finite.
  
  \item The {observation-space} $\bO=\{0,1,2,\hdots,m\}$ is finite of cardinality $|\bO|=m+1$. 
  (With $m=1$, there are only two observations $\bO=\{0,1\}$.
  Such an HMM is referred to as the binary-valued HMM.)
  
  \item The state process $X=\{X_t:t\in\mathbb{T}\}$ is a Markov chain taking values in the state-space $\bS$.  
  Its time-evolution is modeled as
  \begin{align*}
    \sP(X_0=x)  &= \mu(x), \quad x\in\bS. \\
    \sP(X_{t+1} =x' \mid X_t=x) &= A(x,x'), \quad x,x'\in\bS,\quad t=0,1,2\hdots,T-1.
  \end{align*}
  The matrix $A$ is row stochastic, meaning that for each $x\in\bS$, $A(x,\cdot)$ is a probability vector.
  
  \item The  observation process $Z=\{Z_1,Z_2,\hdots,Z_T\}$ takes values in $\bO$. Its model is given by
  \begin{equation*}
    \sP(Z_{t+1} = z \mid X_t=x) = C(x,z),\quad z\in\bO,\;x\in\bS, \quad t=0,1,2\hdots,T-1.
  \end{equation*}
  The matrix $C$ is row stochastic, meaning that for each $x\in\bS$, $C(x,\cdot)$ is a probability vector.
\end{enumerate}
The model is succinctly denoted by $(X,Z)=\text{HMM}(\mu,A,C)$. 

For the purposes of mathematical analysis, the filtrations are
introduced as follows:
  \begin{align*}
    \clF_t&:=\sigma(X_0,Z_1,X_1,\hdots,Z_t,X_t),\quad t\in\mathbb{T},\\
    \clZ_t&:=\sigma(Z_1,Z_2,\hdots,Z_t), \quad t=1,2,\hdots, T,
  \end{align*}
  with $\clZ_0:=\{\phi,\Omega\}$. $\clF=\{\clF_t:t\in\mathbb{T}\}$ is referred
  to as the canonical filtration with respect to which all processes
  are adapted.  The graphical model for the HMM is depicted 
  in~\Fig{fig:hmm}.  

\begin{figure}
  \begin{center}
    \include{figure/hmm.tex}
  \end{center}
  \vspace{-4em}
  \caption{Graphical model for the HMM.
  }
  \label{fig:hmm}
  \vspace*{-10pt}
\end{figure}

\newP{Notation}
The spaces of signed measures on $\bS$ and $\bO$ are denoted by $\clM(\bS)$ and $\clM(\bO)$, respectively, and the set of probability measures $\clP(\bS)\subset \clM(\bS)$ and $\clP(\bO) \subset \clM(\bO)$.  
The space of signed-measures and functions on $\bS$ are both identified with $\Re^d$: a real-valued function $f$ (resp., a measure $\mu$) is identified with a column vector in $\Re^d$ where the $x^{\text{th}}$ element of the vector represents $f(x)$ (resp., $\mu(x)$), for $x\in\bS$, and $\mu(f) := \mu^\tp f = \sum_{x} \mu(x) f(x)$.   
Lower case symbol, e.g., $f$, is used to denote a deterministic function while an upper case symbol, e.g., $F$, is used to denote a random function. 
For such a function, the notation $F\in\clZ_t$ means $F(x)$ is $\clZ_t$-measurable for each $x\in \bS$. 
A product of two functions $f$ and $g$ is denoted by $fg$ ($(fg)(x)=f(x)g(x)$ for $x\in\bS$), and $f^2=ff$.
The unity function is denoted by $\ones$ (e.g., as a function on $\bS$, $\ones (x)=1$ for all $x\in\bS$).
For a vector $f\in\Re^d$, $\text{diag}(f)$ is $d\times d$ diagonal matrix with entries given by the components of $f$.
We follow the convention $\frac{0}{0}=0$. 

\newP{Nonlinear filter} 
The objective of nonlinear (or stochastic) filtering is to compute the conditional measure, also called the posterior. 
It is defined as a conditional expectation,
\begin{equation*}
  \pi_t(f):= \E\big(f(X_t)\mid \clZ_t\big),\quad f \in \Re^d,\quad t\in\mathbb{T}.
\end{equation*}
A recursive formula for the same is given by
\begin{equation*}
  \pi_{t+1}(f) =  \frac{\pi_t(\text{diag}(C(:,Z_{t+1})) Af)}{\pi_t(\text{diag}(C(:,Z_{t+1}))\ones)},\quad t=0,1,2\hdots,T-1,\quad \pi_0=\mu.  
\end{equation*}
The formula is referred to as the nonlinear filter, also called the forward algorithm.  
The convention $\frac{0}{0}=0$ is used to handle the case where the denominator is zero (note in which case the numerator must also be zero).

For the prediction problem, the conditional probability is denoted by
\begin{equation*}\label{eq:prediction_problem}
p_{t}(z):= \sP(Z_{t+1}=z \mid Z_1,Z_2,\hdots,Z_t), \quad t\in\mathbb{T},
\end{equation*}
where note $p_0(z) = \sP(Z_{1}=z)$.  
It is computed from the nonlinear filter as
\begin{equation*}
p_t = \pi_t(C),\quad \text{that is}, \quad p_{t}(z) = \sum_{x\in\bS} \pi_t(x) C(x,z),\quad z\in\bO,\quad t\in\mathbb{T}.
\end{equation*}
We denote
\begin{align*}
  \pi &:= \{\pi_t:t\in \mathbb{T}\},\quad 
  p := \{p_t: t\in \mathbb{T}\}.
\end{align*}
The stochastic processes $\pi$ and $p$ take values in $\clP(\bS)$ and
$\clP(\bO)$, respectively.


\subsection{Problem statement (theory)}

Define a vector-valued mapping $e:\bO\to \Re^m$ as follows:
\begin{equation*}
  e(1) = \begin{bmatrix} 1 \\ 0 \\ \vdots \\ 0 \end{bmatrix}_{m\times 1},\quad e(2) = \begin{bmatrix} 0 \\ 1\\ \vdots\\ 0 \end{bmatrix}_{m\times 1}, \quad \hdots \quad e(m)= \begin{bmatrix} 0 \\ 0 \\ \vdots \\ 1 \end{bmatrix}_{m\times 1}, \quad e(0) = -e(1) -e(2) - \hdots -e(m)= \begin{bmatrix} -1 \\ -1\\ \vdots \\ -1 \end{bmatrix}_{m\times 1}.
\end{equation*}
(Recall here that the cardinality $|\bO| = m+1$.)  

\begin{example}[m=1]
  Consider the HMM with binary-valued observations where recall $\bO =
  \{0,1\}$.  For this model,
  \begin{equation*}
    e(1) = 1,\quad e(0) = -1.
  \end{equation*}
\end{example}

Our theory goal is to derive the following representation for the
conditional measure:
\begin{equation}\label{eq:nonlin_predictor_rep}
  \pi_T(F) = \text{(constant)} - \sum_{t=1}^{T} U_{t-1}^\transpose e(Z_t),\quad \sP\text{-a.s.},\quad F\in\clZ_T,
\end{equation}
where $U=\{U_0,U_1,\hdots,U_{T-1}\}$ is a $\clZ$-adapted stochastic process taking values in $\Re^m$.  
The representation in~\eqref{eq:nonlin_predictor_rep} is referred to as a {\em nonlinear predictor}.

The following remarks draw an analogy, from an input-output perspective, to the Wiener filter and the transformer.

\begin{remark}[Linear vs nonlinear predictors]
  \label{remark:connection_to_transformer}
  Compare~\eqref{eq:nonlin_predictor_rep} with the representation for the linear predictor. While the weights $u$ in a linear predictor are deterministic, the weights in a nonlinear predictor are random---i.e., 
  $U_t$ is allowed to depend upon past observations $\{Z_1,Z_2,\hdots,Z_t\}$ for each $0\leq t\leq T-1$.
  This dependence is what makes the predictor nonlinear. 
\end{remark}

\begin{remark}[Input (tokens)]
  The mapping $e:\bO\to \Re^m$ is reminiscent of the ``one-hot encoding'' of tokens, and may be regarded as such.  
  It differs slightly, however, because here the vocabulary size is $|\bO|=m+1$.
  The form of $e(\cdot)$ is chosen for well-posedness of the representation in~\eqref{eq:nonlin_predictor_rep}  (\Prop{prop:existence_nonlin_predictor_rep}).    
\end{remark}

\begin{remark}[Output (prediction at time $T$)]
  Taking the function $F$ as $F(x)=C(x,z)$ yields a nonlinear predictor for $p_T(z)$, for $z\in\bO$. 
  The key point is that the prediction is computed using the representation~\eqref{eq:nonlin_predictor_rep}.
\end{remark}

\begin{remark}[Predictions for intermediate times]
  Because $U$ is $\clZ$-adapted, the partial sum of the expression on the right-hand side of~\eqref{eq:nonlin_predictor_rep},
  \begin{equation*}
    S_t=\text{(constant)} - \sum_{s=1}^{t} U_{s-1}^\transpose e(Z_{s}),\quad \text{at any intermediate time} \;\; t=1,2,\hdots,T,
  \end{equation*}
  is $\clZ_t$-measurable (i.e., depends only on the observation sequence $\{Z_1,Z_2,\hdots,Z_t\}$ up to time $t$). 
  In this paper, the partial sum is related to the conditional measure
  $\pi_t(\cdot)$, which is helpful to establish a connection to the
  decoder-only transformer where causal predictions are computed for
  $1\leq t\leq T$. 
  \label{rem:intermediate_sum}
\end{remark}

\subsection{Contributions of this paper}

The main contributions of our paper are as follows:
\begin{enumerate}
  \item We begin by proving a well-posedness (existence) result (for $U$) such that the
  representation in~\eqref{eq:nonlin_predictor_rep} holds (\Prop{prop:existence_nonlin_predictor_rep}).  
  The remainder of the paper is concerned with the design and numerical approximation of this $U$.  
  The two objectives inform the split of the paper---theory for design in \Sec{sec:theory} and algorithms for approximation in \Sec{sec:algorithm}. 

  \item The theoretical contribution is an explicit formula for $U$ (\Thm{thm:optimal-solution}). 
  The formula is derived using an original optimal control approach, referred to as the {\em duality theory} for HMMs. 
  A precise statement of the duality---between nonlinear filtering and optimal control---is given in the form of a duality principle (\Thm{thm:duality-principle}).
  Based on this, $U$ is shown to admit an interpretation as an optimal control input (\Prop{prop:nonlinear_predictor_optimal_control}).  
  
  \item The formula for $U$ yields a fixed-point representation of the 
conditional measure $\pi$. This in turn suggests a natural definition 
for a (transformer-like) layer transformation ${\cal N}^\dfltr$ on the space of 
probability measures (\Fig{fig:transformer_objective} in 
\Sec{sec:correspondence_layer_ops}).

\item The algorithmic contribution is the dual filter algorithm, which 
approximates $\pi$ by iterative application of ${\cal N}^\dfltr$ 
(\Sec{sec:algo_for_dual_filter}). The algorithm has complexity $O(d^2T)$, 
which is independent of the vocabulary size $m$ and therefore scales 
efficiently to large vocabularies.
   
  \item The paper closes with numerics. 
  While we do not include learning, the dimension of HMM is chosen to mimic the so-called `GPU parameters' used in certain models of transformers ($d=384$, $m=65$, and $T=256$)~\citep{Karpathy2022}.
\end{enumerate}
The paper makes original contributions to {\em both} nonlinear filtering and control theory:
\begin{enumerate}
  \item The representation~\eqref{eq:nonlin_predictor_rep} for causal nonlinear predictor is new.
  \item Duality theory for HMMs is an original contribution of this paper, and resolves a longstanding open problem in control theory.  
  This is explained as part of \Sec{sec:related_literature}. 
\end{enumerate}
In order to relate our work to transformers, a self-contained
description of the decoder-only transformer is included
(\Sec{sec:main_xfer}) together with a discussion of correspondences
with the dual filter (\Sec{sec:correspondence_layer_ops}). 

\begin{remark}[Inference and learning]
  \label{rem:inference_learning}
  The mathematics of the representation~\eqref{eq:nonlin_predictor_rep}, derived here from
  first principles, provides a generalization of (Wiener filter-like) linear predictors to (transformer-like) nonlinear predictors. 
  The HMM is a mathematically tractable model class for which this derivation
  is carried out. 
  Solving the associated inference problem optimally is seen as the first step towards understanding how to learn these representations from data.
\end{remark}

\subsection{Relevant literature}

In contrast to the time-ordered structure of an HMM, the operations in a transformer are designed to exhibit a permutation symmetry: shuffling the past data (up to time $t$) does not affect the prediction at time $t$ (Remark~\ref{rem:symmetry}).
For this reason, many studies view the `states' in a transformer as (exchangeable) particles interacting through the attention mechanism, which has a certain interpretation as a nonlinear expectation.
This perspective informs the modeling of transformer dynamics across layers~\citep{yun2019transformers,vuckovic2020mathematical,sander2022sinkformers}.
Taking a continuous approximation of the discrete layers leads to an interacting particle ordinary differential equation (ODE) model of the transformer.
For the analysis of such ODE models, it is natural to adopt a mean-field viewpoint and study a continuity equation on the space of probability measures~\citep{geshkovski2023mathematical,geshkovski2024measure,abella2024asymptotic,adu2024approximate,castin2025unified}.

Our objective is not to construct a mathematical model of a transformer, although this remains an important project. 
Rather, our interest lies in deriving, from first principles, transformer-like architectures that solve the prediction problem for which the transformer is designed.  Specifically, our analysis leads to the dual-filter algorithm. We explicitly relate this algorithm to: 
\begin{enumerate}
  \item the transformer, including both its architecture and its attention mechanism (\Sec{sec:main_xfer}), and
  \item the mathematical formalisms, specifically around modeling of a transformer as a transport on the space of probability measures, developed in prior mathematical frameworks (Remark~\ref{rem:IPS}).
\end{enumerate}
Many authors have considered architectures that combine aspects of transformers with HMMs and filtering; see, e.g.,~\cite{liu2022masked,tang2021probabilistic,rohekar2023causal,zhang2023hidden,azeraf2021introducing,wang-etal-2018-neural-hidden,wang2021transformer,goel2024can, chen2026belief}. 
Also relevant are papers that adopt a Bayesian viewpoint~\citep{xie2021explanation,han2023explaining,he2024law,ziemann2024state,ren2021combiner,ildiz2024self}, as well as those that explore control-theoretic properties of the transformer~\citep{kong2024aligning,soatto2023taming,liuobservability,bhargava2023s,luo2023prompt}.
Our approach differs from these works, which focus on interpreting and refining attention mechanisms, rather than modeling the prediction problem from first principles based on the representation in~\eqref{eq:nonlin_predictor_rep}.

There is a large body of literature on variational methods for statistical inference for graphical models including HMMs~\citep{jordan1999introduction,barber2011inference,mcgoff2015statistical,bruna2025survey}. 
The optimal control-type variational problems introduced here are of fundamentally
different  nature.  
The  differences are discussed in the main body of this paper (\Sec{sec:related_literature}) after the optimal control problem has been formally introduced.

\subsection{Outline of the remainder of this paper}

The remainder of this paper is organized as follows: 
\Sec{sec:theory} presents the optimal control approach, beginning with the well-posedness of the representation~\eqref{eq:nonlin_predictor_rep} and concluding with the statement of an explicit formula for $U$. 
\Sec{sec:algorithm} presents the fixed-point equation and its solution
using the dual filter algorithm. 
\Sec{sec:xformer} relates the dual filter to the decoder-only transformer and to prior measure-transport formalisms. 
\Sec{sec:numerics} presents numerical results, and \Sec{sec:conc} concludes with a summary and directions for future work. 
All proofs are collected in the Appendix.


\section{Optimal Control Theory: Formula for $U$}
\label{sec:theory}

Concerning the representation~\eqref{eq:nonlin_predictor_rep}, we begin with a well-posedness result.

\begin{proposition}\label{prop:existence_nonlin_predictor_rep}
  For each $F\in\clZ_T$ there exists a $\clZ$-adapted process $U$ such that~\eqref{eq:nonlin_predictor_rep} holds. 
\end{proposition}

\begin{proof}
  See Appendix~\ref{proof:prop:existence_nonlin_predictor_rep} where additional discussion is included concerning uniqueness of $U$ (uniqueness requires additional assumptions on the HMM $(\mu, A, C)$).
\end{proof}

Our goal in this section is to describe an explicit formula for $U$. 
For this purpose, recall first that the conditional expectation has the following interpretation as the solution of the minimum mean-squared-error (MMSE) optimization problem:
\begin{equation*}\label{eq:minimum-variance}
  \E(|F(X_T) - \pi_T(F)|^2) = \min\{\E(|F(X_T) - S|^2) : S\in \clZ_T\}.
\end{equation*}
The technical considerations of this section are based on expressing the MMSE
objective as an optimal control objective.  We begin with some
preliminaries.

\begin{table}[t]
  \begin{tabularx}{\textwidth}{l l}
    \toprule
    Symbol & Description \\
    \midrule
    Spaces & \\
    $\mathbb{T}$ & Time-horizon: $\{0,1,2,\hdots,T\}$ with $T<\infty$. \\
    $\bS$ & State space: Finite set $\{1,2,\hdots,d\}$. \\
    $\bO$ & Observation space (vocabulary): Finite set $\{0,1,2,\hdots,m\}$. \\
    \midrule 
    Model parameters & \\
    $\mu$ & Prior ($d\times 1$ probability vector). \\
    $A$ & Markov transition matrix (row stochastic $d\times d$ matrix). \\
    $C$ & Emission matrix (row stochastic $d\times (m+1)$ matrix). \\
      \midrule 
    Other parameters & \\
    $c$ & $c(x)$ is a $m\times 1$ vector for each $x\in\bS$ (it is obtained
          from $C$). \\
    $\Gamma$ & $\Gamma:\Re^d\to\Re^d$ (it is obtained
          from $A$).\\
    $R$ & $R(x)$ is a $m\times m$ matrix for each $x\in\bS$  (it is obtained
          from $C$).\\
    \bottomrule
  \end{tabularx}
  \vspace*{-10pt}
  \caption{Nomenclature for the HMM.}
  \label{tab:n_hmmx}
  \vspace*{-10pt}
\end{table}

\subsection{Preliminaries: Martingale increment processes}
\label{sec:prelim_mg}

Define
\begin{equation*}
c(x) := \begin{bmatrix} C(x,1) - C(x,0) \\ C(x,2) - C(x,0) \\ \vdots \\
  C(x,m) - C(x,0) \end{bmatrix}_{m\times 1},\quad x\in\bS.
\end{equation*}
For each fixed $x\in \bS$, $c(x)$ is a $m\times 1$ vector.
The notation is useful to introduce the two $\clF$-martingale increment processes associated with $(X,e(Z))$ as follows:
\begin{align*}
  B_{t+1}(f) &:= f(X_{t+1}) - (Af)(X_t), \quad f\in \Re^d, \quad t=0,1,2\hdots,T-1,\\
  W_{t+1} &:= e(Z_{t+1}) - c(X_t), \quad t=0,1,2\hdots,T-1.
\end{align*}
The time indexing is used so that the processes are adapted with
respect to $\clF$ (e.g., $W_{t} \in \clF_{t}$ and $B_{t} \in
\clF_{t}$). Note we do not use the customary ``$\Delta$'' notation to
denote the increment processes.  The notation is consistent with the
model considered in~\cite[Chapter~2]{elliott2008hidden}.

For these processes, the formulae for conditional moments are given in the following proposition: 
\begin{proposition}\label{prop:B_W_Gamma_R}
  Consider $B$ and $W$. Then
  \begin{align*}
    \E(B_{t+1}(f) \mid \clF_{t}\vee \clZ_{t+1}) &=0, \quad  \E(|B_{t+1}(f)|^2\mid\clF_{t}\vee \clZ_{t+1}) = (\Gamma f)(X_t),\;\; \sP\text{-a.s.},\;\; f\in\Re^d , \;\; t=0,1,2\hdots,T-1,\\[5pt]
    \E(W_{t+1} \mid \clF_t) &=0,\quad \E(W_{t+1} W_{t+1}^\transpose \mid \clF_t) = R(X_t) , \quad \sP\text{-a.s.}, \quad t=0,1,2\hdots,T-1,
  \end{align*}
  where
  \begin{align*}
    (\Gamma f)(x) &:= \sum_{y\in\bS} A(x,y) f^2(y) - (Af)^2(x),\quad x\in\bS, \\
    R(x) &:=  \text{diag}(c(x)) + C(x,0) (I+\ones \ones^\transpose)  - c(x) c^\transpose (x), \quad x\in\bS.
  \end{align*}
\end{proposition}
\begin{proof}
 See Appendix~\ref{appdx:formula_for_R}.
\end{proof}


For reference, the nomenclature for the various HMM-related parameters is summarized in Table~\ref{tab:n_hmmx}.

\begin{table}[t]
  \begin{tabularx}{\textwidth}{l l}
    \toprule
    Symbol & Description \\
    \midrule
    $C$ & $C$ is $d\times 2$ row stochastic matrix  (same notation for the general $m$ case).\\
    $c^+$ & $c^+(x) = C(x,1)$ for each $x\in\bS$. \\
    $c^-$ & $c^-(x) = C(x,0)$ for each $x\in\bS$. \\
    $c$ & $c(x) = c^+(x) - c^-(x)$ for each $x\in\bS$ (same notation for the general $m$ case). \\
    $r$ & $r(x) = (1-c^2(x))$ for each $x\in\bS$ (For $m>1$, this is denoted by $R(x)$).\\
    \bottomrule
  \end{tabularx}
  \vspace*{-10pt}
  \caption{Nomenclature for the binary-valued HMM ($m=1$ and $\bO=\{0,1\}$).}
  \label{tab:n_hmm_binary}
  \vspace*{-10pt}
\end{table}

\begin{example}[m=1]
  Consider the HMM with binary-valued observations where $e(1)=1$ and $e(0)=-1$.  
  For this model, it is convenient to introduce notation
  \begin{equation*}
    c^+(x) := C(x,1),\quad c^{-}(x) := C(x,0),\quad x\in \bS.
  \end{equation*}
  Note, $c^\pm(x)=\sP( e(Z_{t+1})=\pm 1 \mid X_t=x)$. 
  For each fixed $x\in\bS$, $c^+(x)+c^-(x) = 1$.  
  Then 
  \begin{equation*}
    c(x)=C(x,1) - C(x,0) = c^+(x)-c^-(x),\quad x\in\bS.
  \end{equation*}
  Likewise, $R(x)$ is now a scalar.
  The scalar is denoted by $R(x)=r(x)$, and given by
  \begin{equation*}
    r(x)= C(x,1) + C(x,0) - c^2(x)= c^+(x) + c^{-}(x) - c^2(x) = (1-c^2(x)),\quad x\in\bS.
  \end{equation*}
  The nomenclature for the binary-valued HMM is included in
  Table~\ref{tab:n_hmm_binary}. 
\end{example}


\begin{table}[t]
  \begin{tabularx}{\textwidth}{l l}
    \toprule
    Symbol & Description \\
    \midrule
    HMM &\\
    $X$ & State process: $\{X_0,X_1,\hdots,X_T\}$ is $\bS$-valued. \\
    $Z$ & Observation process: $\{Z_1,Z_2,\hdots,Z_T\}$ is $\bO$-valued.\\
    $B$ & Process noise: $\{B_1,B_2,\hdots,B_T\}$ is
          $\Re^d$-valued.\\
        $W$ & Observation noise: $\{W_1,W_2,\hdots,W_T\}$ is $\Re^m$-valued.\\
    \midrule
    Nonlinear filter &\\
    $\pi$ & Conditional measure:
            $\{\pi_0,\pi_1,\pi_2,\hdots,\pi_T\}$ with $\pi_0=\mu$ is $\clP(\bS)$-valued. \\
    $p$ & Prediction: $\{p_1,p_2,\hdots,p_T\}$ is $\clP(\bO)$-valued.
    \\
        \midrule
    BS$\Delta$E &\\
    $U$ & Control input: $\{U_0,U_1,\hdots,U_{T-1}\}$ is
          $\Re^m$-valued.\\
    $(Y,V)$ & Solution pair: $\{Y_0,Y_1,\hdots,Y_{T-1},Y_T\}$ with $Y_T=F$ is
              $\Re^d$-valued.\\
    & $\qquad \qquad \qquad \{V_0,V_1,\hdots,V_{T-1}\}$ is
              $\Re^{m\times d}$-valued.\\
    \bottomrule
  \end{tabularx}
  \vspace*{-10pt}
  \caption{Nomenclature for the stochastic processes.}
  \label{tab:nomenclature_sp}
  \vspace*{-10pt}
\end{table}

\subsection{Function spaces}

The function spaces are as follows:
\begin{align*}
  \clY & :=\{Y:\Omega\times \mathbb{T} \to \Re^d, \;\;
       Y_t\in\clZ_t, \;t\in\mathbb{T} \},\\
  \clV & :=\{V:\Omega \times \{0,1,2,\hdots,T-1\}\to \Re^{m\times d}, \;\;V_t\in\clZ_t, \;\;0\leq t\leq T-1 \},\\
  \clU & :=\{U:\Omega\times \{0,1,2,\hdots,T-1\} \to \Re^m, \;\;U_t\in\clZ_t, \;0\leq t\leq T-1\}.
\end{align*}
$\clY$ is the space of function-valued $\clZ$-adapted stochastic
processes (at time $t$, $Y_t$ is a random function on $\bS$).  
Likewise, $\clV$ is the space of matrix-valued $\clZ$-adapted stochastic processes (at time $t$, each row of the matrix $V_t$ is a random function on $\bS$).  
$\clU$ is the space of admissible control inputs. 
(It is noted that $\clU$ is defined only for time indices
$t=0,1,2,\hdots,T-1$, whereby, in~\eqref{eq:nonlin_predictor_rep}, the control $U_{t-1}$ is the weight for
the $t$-th observation $e(Z_{t})$, for a total of $T$-many weights for
$T$-many observations.)

Based on \Prop{prop:existence_nonlin_predictor_rep}, our goal is to find an element $U\in \clU$ such that the representation~\eqref{eq:nonlin_predictor_rep} holds.  
To this end, we introduce an optimal control problem that involves two key components: 
\begin{enumerate} 
  \item A definition of the dynamic constraint, and
  \item A definition of the optimal control objective.
\end{enumerate}
These are described in detail below, followed by a statement of the duality principle that links the filtering objective (MMSE) to the optimal control problem.

\subsection{Dual optimal control problem}
\newP{1. Dynamic constraint}~is a backward stochastic difference equation  (BS$\Delta$E) 
\begin{subequations}\label{eq:dual_BSDE}
  \begin{align}
    Y_t(x) &= (AY_{t+1})(x) + c^\transpose(x) (U_t + V_t(x)) - V_t^\transpose(x)
            e(Z_{t+1}),
      \quad \;\; x\in\bS,\quad t=0,1,2,\hdots,T-1, \label{eq:dual_BSDE_a}\\
    Y_T(x)  & = F(x),\quad x\in\bS.  \label{eq:dual_BSDE_b}
  \end{align}
\end{subequations}
Here $U:=\{U_t:0\leq t\leq T-1\} \in \clU$ is an $\Re^m$-valued stochastic process referred to as the control input and $F\in \clZ_T$ is a $\Re^d$-valued random vector referred to as the terminal condition.  
For a given control input $U\in \clU$ and the terminal condition $F\in \clZ_T$, the problem is to obtain a pair of $\clZ$-adapted stochastic processes,
\begin{align*}
  Y&:=\{Y_t(x): Y_t(x)\in\clZ_t, x\in\bS,  0\leq t\leq T\},\\
  V&:=\{V_t(x): V_t(x)\in\clZ_t, x\in\bS,  0\leq t\leq T-1\},
\end{align*}
where for each fixed $x\in\bS$, $Y_t(x)$ is real-valued and $V_t(x)$ is $\Re^m$-valued.
The pair is denoted by $(Y,V)$ and referred to as the solution of the
BS$\Delta$E~\eqref{eq:dual_BSDE}.
As stated,~\eqref{eq:dual_BSDE} is an example of a BS$\Delta$E on a lattice, the general theory for which appears in~\citep{fukasawa2025backward}. 
Based on this theory, we have the following well-posedness result:

\begin{proposition}\label{prop:existence_BSDE}
Suppose $U \in \clU$ and $F\in \clZ_T$.  
Then there exists a unique $(Y,V) \in \clY \times \clV$ that solves~\eqref{eq:dual_BSDE}. 
\end{proposition}
\medskip
\begin{proof}
  See Appendix~\ref{proof:prop:existence_BSDE}.  
  The proof is self-contained based on adapting theory from~\citep{fukasawa2025backward}. 
\end{proof}

For reference, the nomenclature for the three different types of stochastic
processes introduced thus far is given in Table~\ref{tab:nomenclature_sp}.

\newP{2. Optimal control objective}~Fix $U\in \clU$.  
For the solution $(Y,V)$ of the BS$\Delta$E~\eqref{eq:dual_BSDE}, define 
\begin{align*}
  \sJ_T(U;F)  := \dvar (Y_0(X_0)) + \E \Big( \sum_{t=0}^{T-1} l (Y_{t+1},V_t,U_t\,;X_t)  \Big),
\end{align*}
where $\dvar(Y_0(X_0))=\E(|Y_0(X_0) - \mu(Y_0)|^2) = \mu(Y_0^2)-\mu(Y_0)^2$ (note here $Y_0$ is a deterministic function), and the running cost $l:\Re^d\times\Re^{m\times d}\times\Re^m\times\bS\to\Re$ is given by,
\begin{equation*}\label{eq:running_cost_formula}
  l(y,v,u;x):= (\Gamma y)(x) + (u+v(x))^\transpose R(x) (u+v(x)),\quad y\in\Re^d,\;v\in\Re^{m\times d},\;u\in\Re^m,\;x\in\bS.
\end{equation*}
Here, $v(x)$ is the $x$-th column vector of the $m\times d$ matrix $v=\begin{bmatrix}v(1) & \hdots & v(x) & \hdots & v(d) \end{bmatrix}_{m\times d}$.

\medskip

Now that the dynamic constraint and the optimal control objective have been defined, the duality principle is given in the following theorem.

\begin{theorem}[Duality principle]\label{thm:duality-principle}
	Let $U\in \clU$ and $F\in\clZ_T$.  
  Consider an estimator
	\begin{equation*}\label{eq:estimator}
		S_T := \mu(Y_0) - \sum_{t=1}^{T} U_{t-1}^\transpose e(Z_{t}).
	\end{equation*}
	Then 
	\begin{equation*}\label{eq:duality-principle}
		\sJ_T(U;F) = \E\big(|F(X_T)-S_T|^2\big).
	\end{equation*}	
\end{theorem}

\medskip
      
\begin{proof}
  See Appendix~\ref{appdx:duality-principle}.
\end{proof}

Noting that the right-hand side is the mean-squared error, the duality principle provides an optimal control approach to compute the conditional expectation. 

\begin{itemize}
  \item \textbf{Dual optimal control problem (OCP):}
\end{itemize}
\begin{equation}\label{eq:dual-optimal-control}
  \min_{U \in \clU} \; \sJ_T(U;F) 
  \quad \text{subject to} \quad \eqref{eq:dual_BSDE}
\end{equation}

The following proposition is a corollary of \Prop{prop:existence_nonlin_predictor_rep} and helpful to relate the optimal control input, solution to~\eqref{eq:dual-optimal-control}, to the desired representation~\eqref{eq:nonlin_predictor_rep}.  
  
\begin{proposition} \label{prop:nonlinear_predictor_optimal_control}
  Consider OCP~\eqref{eq:dual-optimal-control}.  
  For this problem, there exists an optimal control $U^\opt = \{U_t^\opt:0\le t \le T-1\} \in \clU$ such that
  \begin{align*}
    \pi_T(F) = \mu(Y_0^\opt) - \sum_{t=1}^{T} (U_{t-1}^\opt)^\tp e(Z_{t}),\quad \sP\text{-a.s.},
  \end{align*}
  where $Y_0^\opt$ is obtained from solving the BS$\Delta$E~\eqref{eq:dual_BSDE} with $U=U^\opt$.  
  The optimal value is given by
  \begin{equation*}
    \sJ_T(U^\opt;F) = \E(|F(X_T)-\pi_T(F)|^2) = \text{MMSE}.
  \end{equation*}
\end{proposition}
\begin{proof}
  See Appendix~\ref{appdx:nonlinear_predictor_optimal_control} where Remark~\ref{rem:uniqueness_of_U} is included concerning uniqueness of $U^\opt$.  
\end{proof}

\subsection{Formula for optimal control}

\newP{Notation} The following notation is introduced to denote the formula for optimal control:
\begin{equation*}
  \phi(y,v;\rho) := - \rho(R)^{\dagger} \left( \rho ((c-\rho (c))y)- \rho (Rv)\right), \quad y\in\Re^d, \; v\in \Re^{m\times d},\; \rho\in\clP(\bS).
\end{equation*}
Here, $\rho(R)^{\dagger}$ denotes the pseudo-inverse of the $m\times m$ matrix $\rho(R):=\sum_{x\in\bS} \rho(x) R(x)$. 
The other two terms are $\rho ((c-\rho (c))y) := \sum_{x\in\bS} \rho(x) (c(x) - \rho(c)) y(x)$ and $\rho (Rv):= \sum_{x\in\bS} \rho(x) R(x)v(x)$.  
Note that both of these terms are $m\times 1$ vectors.

Let $u\in \Re^m$ and $q\in \clP(\bO)$. Denote
\begin{equation*}
  \langle u,u \rangle_q:= q(0)(-1^\tp u)^2 + \sum_{i=1}^m q(i) (u(i))^2 -\left( q(0) (-1^\tp u) + \sum_{i=1}^mq(i) u(i) \right)^2.
\end{equation*}
The right-hand side is the variance of the function $\begin{bmatrix}(-1^\tp u) \\ u\end{bmatrix} \in \Re^{m+1}$ with respect to the probability vector $q$.
By Jensen's inequality, $\langle u,u \rangle_q=0$ implies $u(z) =
\text{constant}$ for all such $z\in\bO$ with $q(z)>0$. 
If moreover $q(0)>0$, then the constant must be zero.

\begin{theorem}\label{thm:optimal-solution}
	Consider the OCP~\eqref{eq:dual-optimal-control}. 
  Then an optimal control is of the feedback form given by
  \begin{subequations}
	  \begin{equation}\label{eq:optimal_control_formula}
      U_t^\opt = \phi(Y_t,V_t;\pi_t),\quad \sP\text{-a.s.},\quad 0\leq t\leq T-1.
    \end{equation}
    For any $U\in\clU$,
    \begin{equation}\label{eq:optimal_control_value}
      \sJ_T(U;F) = \underbrace{\E (|F(X_T) - \pi_T(F)|^2)}_{\emph{MMSE}}+\E \left( \sum_{t=0}^{T-1}
      \langle (U_t-U_t^\opt), (U_t-U_t^\opt) \rangle_{p_t}\right),
    \end{equation}
    where $p_t=\pi_t(C)$. 
    Suppose $(Y^\opt,V^\opt)$ is the solution of the BS$\Delta$E~\eqref{eq:dual_BSDE} with $U=U^\opt$.  
    Then the following representation holds:
		\begin{align}
		  \pi_t (Y_t^\opt) = \mu (Y_0^\opt) - \sum_{s=1}^{t} (U_{s-1}^\opt)^\tp e(Z_{s}),\quad \sP\text{-a.s.},\quad 1\leq t\leq T.
      \label{eq:estimator-t}
		\end{align}
	\end{subequations}
\end{theorem}
\medskip
\begin{proof}
  See Appendix~\ref{appdx:proof_of_optimal-solution}. 
  Concerning uniqueness, the formula~\eqref{eq:optimal_control_formula} identifies {\em one} optimal control input and~\eqref{eq:optimal_control_value} provides a characterization of {\em all} such $U$ for which the representation in~\eqref{eq:nonlin_predictor_rep} holds.
  The optimal control is unique iff $p_t(z) > 0$, $\sP$-a.s., $\forall\,z\in \bO$, and $\forall\,t\in\mathbb{T}$.
  A sufficient condition is $\min\{C(x,z) \,:\,x\in\bS,\,z\in\bO\} >0$ (see also Remark~\ref{rem:uniqueness_of_U} in Appendix).
\end{proof}

The most significant implication of this theorem is captured in the following remark:

\begin{remark}[Fixed-point representation of conditional measure $\pi$]
  Compare~\eqref{eq:estimator-t} with the representation given in 
  \Prop{prop:nonlinear_predictor_optimal_control}.
  While the representation in \Prop{prop:nonlinear_predictor_optimal_control} 
  is given only for the terminal time $T$, \Thm{thm:optimal-solution} shows 
  that the optimal control input yields a causal representation for the entire 
  conditional measure $\pi=\{\pi_t:0\leq t\leq T\}$.  
  Moreover, because the optimal control $U_t^\opt$ is a function of $\pi_t$ 
  (via the formula~\eqref{eq:optimal_control_formula}),~\eqref{eq:estimator-t} 
  is a fixed-point representation of $\pi$. 
\end{remark}
                    
There are three problems in directly applying~\eqref{eq:estimator-t} for computational
purposes:
\begin{enumerate}
\item To compute the optimal control input, one needs to first compute
  $\pi$.  This, however,
  somewhat undermines the original objective, as the primary goal of
  the algorithm is to compute $\pi$!
\item The solution procedure involves solving the BS$\Delta$E as an
  intermediate step. However, finding numerical solutions to the
  BS$\Delta$E is a challenging task, even in low-dimensional
  settings. 
\item In applications of current interest, $m$ is often large.  For
  instance, in transformer models, a typical vocabulary size is $m\approx
  100 K$. The computation of optimal control using the
  formula for $\phi$ has a complexity $O(m^3$), which becomes prohibitive for large $m$.
\end{enumerate}
An analysis is presented in the following section that resolves these
three challenges and yields a practical algorithm.  For this purpose,
the following corollary of~\Thm{thm:optimal-solution} is useful.

\begin{corollary}\label{prop:optimal_control_formula_T=1}
  Consider the OCP~\eqref{eq:dual-optimal-control} for $m=1$ and $T=1$.  Then 
  \begin{equation}\label{eq:optimal_control_formula_T=1}
    U_0^\opt = \begin{cases}
      \frac{-1}{(1-\mu(c)^2)} \left( \mu((Af)(c-\mu(c))) - (\mu(A\tilde f) - \mu((A\tilde f)c) \mu(c))\right), & 1-\mu(c)^2 \neq 0, \\
      0 & \text{o.w.},
    \end{cases} 
  \end{equation}
  where $f(x) = \frac{F^+(x)+F^-(x)}{2}$ and $\tilde f(x) = \frac{F^+(x)-F^-(x)}{2}$.
  For any admissible (deterministic) value of $U_0$,
  \begin{align*}
    \sJ_1(U_0;F) = \underbrace{\E (|F(X_1) - \pi_1(F)|^2)}_{\emph{MMSE}} +(1-\mu(c)^2) (U_0-U_0^\opt)^2.
  \end{align*}
\end{corollary}

     \begin{proof}
       See Appendix~\ref{sec:simple_case}.  
       Although stated here as a corollary, the
       Appendix contains a self-contained formulation of the OCP and derivation of the formula~\eqref{eq:optimal_control_formula_T=1}.  The derivation is useful to obtain an intuitive
       understanding of the general result in its simplest possible setting.  
     \end{proof}

     \subsection{Related literature}\label{sec:related_literature}

     \paragraph{Relationship to BS$\Delta$E literature:} The form of
     the BS$\Delta$E considered in this paper is borrowed
     from~\cite{fukasawa2025backward}.  The paper introduces the
     coordinates $e:\bO\to\Re^m$ which are referred to as the lattice
     coordinates.  Using these coordinates, 
     existence and uniqueness theory for general types of  BS$\Delta$E
     is given together with its application to
     finance.  The BS$\Delta$E~\eqref{eq:dual_BSDE} is a special case
     of the general form considered in~\cite{fukasawa2025backward}.
     The special form is introduced for the express purpose of
     deriving the nonlinear predictor~\eqref{eq:nonlin_predictor_rep}
     for an HMM.  In particular, both the
     optimal control problem (\Thm{thm:duality-principle}) and its
     solution (\Thm{thm:optimal-solution}) are 
     original contributions of this paper.  

\paragraph{Relationship to variational inference:}
For probabilistic graphical models, a standard way to cast
inference as an optimization problem is through variational inference
(VI), which formulates an entropy-type objective~\citep{jordan1999introduction},~\cite[Sec.~3]{wainwright2008graphical},
and
\cite[Ch.~11]{koller2009probabilistic}:  Let $\mathsf{P}_{X,Z}(\cdot,\cdot)$ denote the joint probability law of the
stochastic process $(X,Z)$ (which may be more general than the HMM
considered here), and let $Z=z$ denote a
fixed realization of the observations. In VI, the goal is to
approximate the conditional law $\sP_{X|Z}(\cdot \mid Z=z)$ by selecting a
law $\sP^\theta_X(\cdot)$ from a parameterized family so as to
minimize the so-called {\em free energy} defined as follows:
\begin{equation*}
\text{Free energy}(\theta)
:= \sum_{x} \sP^\theta_X(x)
\log \frac{\sP^\theta_X(x)}{\sP_{X,Z}(x, z)} .
\end{equation*}
Since $\sP_{X|Z}(x \mid z) = \sP_{X,Z}(x,z)/\sP_Z(z)$, minimizing the
free energy is equivalent (up to the additive constant $\log
\sP_Z(z)$) to minimizing the Kullback--Leibler divergence
$\mathrm{D}(\sP^\theta_X \,\|\, \sP_{X|Z}(\cdot \mid z))$. 
The negative of the free energy is referred to as the {\em evidence
  lower bound (ELBO)}.  Its maximum possible value is given by $\log
\sP_{Z}(z)$.  This maximum value is attained if and only if
$\sP^{\theta^\opt}_X(\cdot)=\sP_{X|Z}(\cdot \mid Z=z)$, in which case VI yields exact inference.

For an HMM, when the variational family is unrestricted over the
hidden state sequence, the ELBO is achieved at its maximum and VI is
exact. In this case, the optimal marginal
distribution at time $t=T$ coincides with the filtering posterior
$\pi_T$.


\paragraph{VI-based dual optimal control formulations:}
For certain classes of HMMs---notably stochastic differential
equation (SDE) models in continuous-time
with additive white noise observations
(see~\cite{xiong2008introduction})---the VI objective based on the
divergence minimization
can be expressed as an optimal control-type objective.
This equivalence is obtained by parameterizing the approximating law
$\sP^\theta(X)$ as the law of a controlled stochastic process; see
\cite[Sec.~3]{mitter2003} for the case where $X$ is an
It\^o diffusion and \cite[Sec.~2.2]{van2006filtering} for the case where
$X$ is a discrete-valued continuous-time Markov process.  
The resulting optimal control problem is a generalization of the
classical Mortensen's maximum-likelihood
estimator~\citep{mortensen1968}, also referred to as minimum
  energy duality~\citep{hijab1980minimum}.  
Related papers on this topic
include~\citep{todorov2008general,sutter2016variational,kim2020smoothing,raginsky2024variational}. 

In contrast, the optimal
control formulation developed in the present work generalizes
Kalman's minimum variance duality, which is based on MMSE
rather than the divergence-type objective. 
The terminology ``minimum variance duality” and ``minimum energy
duality” follows Bensoussan~\cite[p.~180]{bensoussan2018estimation}. As
noted in~\cite[p.~100]{kailath2000linear}, these two duality principles
are not directly related in general, despite yielding identical
solutions in the linear Gaussian model.
Prior to our work~\citep{kim2019duality}, it was widely believed that
minimum variance duality could not be extended beyond the LQG setting;
see, for example,~\cite{todorov2008general}, who wrote:
\begin{quote}\em
``Kalman's duality has been known for half a century and has attracted a
lot of attention. If a straightforward generalization to non-LQG
settings was possible it would have been discovered long ago. Indeed,
we will now show that Kalman's duality, although mathematically sound,
is an artifact of the LQG setting.''
\end{quote}
The present paper provides, to our knowledge, the first extension of
minimum variance duality to discrete-time HMMs with
discrete-valued observations. The derivation is based on suitably
adapting original arguments from our prior
work~\citep{kim2019approach,duality_jrnl_paper_II} given there in the
continuous-time settings of the model.  For this purpose, the key
technical tool employed here is the BS$\Delta$E theory developed
in~\citep{fukasawa2025backward}.  

There are two reasons for considering Kalman's minimum variance duality:
\begin{enumerate}
  \label{enum:kalman_duality}
  \item Kalman's duality is structurally fundamental from a
    control-theoretic perspective.  This is so because the MMSE objective
    is directly tied to the structural duality between observability and controllability~\citep[Ch.~8]{callier2012linear}. 
  In particular, the HMM $(\mu,A,C)$ is observable if and only if the BS$\Delta$E~\eqref{eq:dual_BSDE} is controllable~\citep{duality_jrnl_paper_I}.
  From a theoretical standpoint, this property is useful for the
  purposes of filter stability
  analysis~\citep{kim2024variance,kim2024backward}.  
  \item From a practical standpoint, the OCP~\eqref{eq:dual-optimal-control} is useful to design the optimal weights in the nonlinear predictor~\eqref{eq:nonlin_predictor_rep}.  
\end{enumerate}
A historical overview of duality theory
in control and estimation is provided in~\citep{kim2025arrow}.

\section{Fixed-point equation: Dual filter algorithm}
\label{sec:algorithm}

Consider a fixed sample path of observations ${z}:=[z_1,z_2,\hdots,z_T]
\in \bO^T$ such that $\sP(Z=z)>0$ (such a sample path is an input to a
transformer).  The conditional measure for this fixed sample path is
denoted by
\begin{equation*}
\pi^{(z)} := \left[\pi_1^{(z)}, \pi_2^{(z)}\hdots,
  \pi_{T}^{(z)}\right]\in\clP(\bS)^T  \quad \text{where}\quad  \pi_t^{(z)}(x)=\sP(X_t = x \mid
Z_1=z_1,\hdots,Z_t=z_t),\;\; x\in\bS, \;\; 1\leq t\leq T.
\end{equation*}
Our objective in this section is to define a layer transformation
\begin{equation*}
  {\cal N}^\dfltr:\clP(\bS)^T \mapsto \clP(\bS)^T\quad
  \text{such that}\quad \pi^{(z)} = {\cal N}^\dfltr(\pi^{(z)}).
\end{equation*}
(The $d\times T$ input and output data structures are the same as those
of a transformer layer.)


For the purpose of defining the layer transformation ${\cal N}^\dfltr$, consider first a model for binary-valued observations as follows:
\begin{equation*}
c_t^{+}(x) := C(x, z_{t}),\quad c_t^{-}(x) := 1- C(x, z_{t}),\quad
c_t(x):=c_t^{+}(x) - c_t^{-}(x),\quad x\in\bS, \quad 1\leq t\leq T.
\end{equation*}

Let $\rho = [\rho_1,\rho_2,\hdots,\rho_T]\in \clP(\bS)^{T}$.  Consider
a deterministic backward difference equation (B$\Delta$E) as follows:
\begin{subequations}\label{eq:opt_BDE}
\begin{align}
y_t(x) &= (Ay_{t+1})(x) + c_{t+1} (x) u_t,\quad  x\in\bS,\quad
         t=0,1,2,\hdots,T-1, \label{eq:opt_BDE_a} \\
u_t & = \begin{cases} \phi(y_t,0;\rho_t) ,\quad
  t=1,2,\hdots,T-1,\\
  \phi(y_t,0;\mu),\quad t=0,
\end{cases}
  \label{eq:opt_BDE_a2}
  \\
y_T(x)  & = f(x),\quad x\in\bS, \label{eq:opt_BDE_b}
\end{align}
where the control input $u=\{u_t\in\Re:0\leq t\leq T-1\}$ and the
solution $y=\{y_t(x)\in\Re:x\in\bS,\;0\leq t\leq T\}$ are both 
deterministic processes (the lower-case notation is used to stress
this).  Note that the control input is obtained using the formula
$\phi$ for
optimal control with a given measure $\rho_t$.

Because the terminal condition $y_T=f$ is arbitrary, the solution $y$ defines a measure-valued process $\rho^+ = [\rho_1^+,\rho_2^+,\hdots,\rho_T^+]\in \clM(\bS)^{T}$ as follows:
\begin{align}
  \rho^+_t(y_t) := \mu (y_0) - \sum_{s=1}^{t} u_{s-1},\quad 1\leq t\leq T,\label{eq:opt_BDE_c}
\end{align}
\end{subequations}
where on the right-hand side, $\{u_s\in\Re:0\leq s\leq t-1\}$ is according to~\eqref{eq:opt_BDE_a2}
and~$y_0$ is the solution of~\eqref{eq:opt_BDE_a} at
time $t=0$.  It is clear that $\rho^+_t \in \clM(\bS)$ because the
right-hand side of~\eqref{eq:opt_BDE_c} is well-defined for each
$t=1,2,\hdots,T$.  The layer transformation for the dual filter is
now defined as follows: 
  \begin{equation*}
  {\cal N}^\dfltr: \clD \subset \clP(\bS)^T \mapsto \clP(\bS)^T\quad
  \text{as}\quad {\cal N}^\dfltr(\rho):=\rho^+,
\end{equation*}
where the domain $\clD$ is defined to be the largest subset in $\clP(\bS)^{T} $ such
that $\clN^\dfltr\rho \in \clP(\bS)^{T}$.  That the domain is non-empty is because of the following proposition which also shows the significance of~\eqref{eq:opt_BDE} to the computation of $\pi^{(z)}$.

\begin{proposition}\label{prop:dual_filter}
Consider~\eqref{eq:opt_BDE}.  Then $\pi^{(z)}\in \clD$ and moreover
  \begin{equation*}
  \clN^\dfltr (\pi^{(z)}) = \pi^{(z)}.
  \end{equation*}
\end{proposition}
\begin{proof}
  See Appendix~\ref{appdx:prop_dual_filter}.
\end{proof}

The result lays the foundation for a practical algorithm described next.  At its core, the 
algorithm implements the mapping $\clN^\dfltr$.

\subsection{Algorithms}
\label{sec:algo_for_dual_filter}

In the following, we describe two algorithms to approximate the
solution of the fixed-point equation:
\begin{enumerate}
\item An iterative algorithm.  
\item A single-shot algorithm. 
\end{enumerate}
The iterative algorithm is useful to obtain a correspondence with the transformer.  
The single-shot algorithm is more efficient, and requires only a single iteration.  
The pseudocodes for all the algorithms described in this section appear in Appendix~\ref{appdx:algorithms}.

\newP{1. Iterative algorithm}
\begin{enumerate}
\item 
  Initialize $\rho^{(0)}=\{\rho_t^{(0)}:1\leq t\leq T\}\in\clP(\bS)^T$.  
  Inspired from input used in a transformer, a reasonable choice is 
  \begin{equation*}
    \rho_t^{(0)}(x) = \frac{C(x,z_t)}{\sum_{x'\in \bS} C(x',z_t)} ,\quad x\in\bS, \quad 1\leq t\leq T.
  \end{equation*}
\item 
  Denote $\rho^{(\ell)}=\{\rho_t^{(\ell)}:1\leq t\leq T\}\in \clP(\bS)^T$ for $\ell=1,2,\hdots,L$.  
  The dual filter~\eqref{eq:opt_BDE} is simulated to compute
  \begin{align*}
    \rho^{(\ell+1)} = \text{Project}(\clN^\dfltr \rho^{(\ell)}),\quad \ell=0,1,2,\hdots,L-1,
  \end{align*}
  where $\text{Project}(\cdot)$ is used to ensure that the
  $\rho^{(\ell+1)} \in \clP(\bS)^T$ (see
  Algorithm~\ref{alg:normalize_rho} in Appendix~\ref{appdx:algorithms}).
\item At the final layer, the nonlinear predictor for the next observation is 
  obtained as
  \begin{equation*}
  p_{t}(z)=\sP(Z_{t+1}=z | Z_1=z_1,\hdots,Z_t=z_t) = \rho^{(L)}_t(C)=\sum_{x\in\bS}
  \rho^{(L)}_t (x) C(x,z),\quad t=1,2,\hdots,T,\quad 
  z\in \bO.
\end{equation*}
Note the prediction is over the entire time horizon.  
\end{enumerate}
The pseudocode for the iterative algorithm is given in 
Algorithm~\ref{alg:dual_filter_iterative} in Appendix~\ref{appdx:algorithms}.

%



\newP{2. Single-shot algorithm}  In
contrast to the iterative algorithm, the single-shot algorithm
requires only a single iteration for convergence.  The difference between
the two algorithms is as follows:
\begin{enumerate}
\item For the iterative algorithm, each iteration implements a single
  backward pass.
\item For the single-shot algorithm, there is only a single iteration
  which implements
  $T$ backward passes, where the result of the $t$-th pass is used to
  compute the result of the $(t+1)$-th pass. 
\end{enumerate}
Using the single-shot algorithm, after the $t$-th backward pass, the
conditional measure is computed as follows:
\begin{equation*}
 \pi_t^{({z})}(f)=\mu(y_0) - \sum_{s=1}^{t}u_{s-1},\quad 1\leq t\leq T,
\end{equation*}
where the control $\{u_{s-1}:1\leq s\leq t\}$ are computed using the
formula for the optimal control with measures $\{\mu,\pi_1^{({z})},
\hdots,\pi_{t-1}^{({z})}\}$ (known from $(t-1)$-th pass). 
From this, the prediction at time $t$ is computed as $p_t(z) =
\pi_t^{({z})} (C(:,z))$. 
The pseudocode for the single-shot algorithm is given in 
Algorithm~\ref{alg:dual_filter_single_shot} in Appendix~\ref{appdx:algorithms}.

An advantage of the single-shot algorithm is that the computation does not require multiple
iterations. On the other hand, the iterative algorithm has structural
similarities with the transformer architecture.  These are described
at length in the following subsection.  Note that either of the
algorithms is designed to numerically compute the fixed-point of
$\clN^\dfltr$, based on solving~\eqref{eq:opt_BDE}.

\section{Relationship to transformer modeling}
\label{sec:xformer}

\subsection{Input-output architecture}
\label{sec:main_xfer}

We begin by describing, in a self-contained manner, the architecture
and the algorithmic operations inside a decoder-only transformer.
The description is based
upon~\citep{phuong2022formal}~\cite[Ch.~10]{jm3} and~\citep{raschka2024build}.

In a decoder-only transformer, as a first step, a token is ``embedded'' using a $d\times m$ embedding
matrix denoted in this paper by $C^\xfer$:
\begin{equation*}
z \mapsto C^\xfer (:,z) \in \Re^d,\quad z\in\bO.
\end{equation*}
Using the embedding matrix,
\begin{equation*}
\begin{bmatrix} z_1 & \hdots & z_t & \hdots &
  z_T \end{bmatrix}_{1\times T}  \mapsto \begin{bmatrix} C^\xfer (:,z_1) & \hdots & C^\xfer (:,z_t) & \hdots & C^\xfer (:,z_T) \end{bmatrix}_{d\times T}.
\end{equation*}
The input $C^\xfer (:,z_t)$ is augmented with the so-called positional
encoding as follows:
\begin{equation*}
e_t(x) := C^\xfer(x,z_t) + W_p(x,t),\quad x\in\bS, \quad t=1,2,\hdots,T,
\end{equation*}
where $W_p \in \Re^{d\times T}$ is referred to as the positional encoding matrix.  In
the original transformer paper~\citep{vaswani2017attention}, the rows of $W_p$ are defined
according to the sinusoidal-positional-encoding:
\begin{align*}
  W_p(2i-1,t) = \sin(\ell_{\text{max}}^{-\frac{2i}{d}} t),\quad 
  W_p(2i,t)  = \cos(\ell_{\text{max}}^{-\frac{2i}{d}} t), \quad
  i=1,2,\hdots,\frac{d}{2},\quad t=1,2,\hdots,T,
\end{align*}
where $\ell_{\text{max}} = 10,000$.  The factor
$\ell_{\text{max}}^{-\frac{2i}{d}}$ determines the frequency of
oscillations, ensuring a wide range of scales, as $i$ varies from
$1,2,\hdots,\frac{d}{2}$.  The positional encoding is the {\em only}
mechanism by which information about the position (time) $t$ is introduced
into the transformer. 
Other types of positional encoding are also possible (see Remark~\ref{rem:RPE}).

From an input-output perspective, a decoder-only transformer implements a causal nonlinear
transformation that transforms a $d\times T$ matrix at the
input into a $d\times T$ matrix at the output,
\begin{equation*}
\begin{bmatrix} e_1 & \hdots & e_t & \hdots & e_T \end{bmatrix}_{d\times T} \mapsto \begin{bmatrix} \sigma_1^{(L)} & \hdots & \sigma_t^{(L)} & \hdots & \sigma_T^{(L)} \end{bmatrix}_{d\times T}.
\end{equation*}
From the output $\sigma_t^{(L)}$, the conditional probability of the
next token is computed as follows: 
\begin{equation*}
p_{t}(z)=\sP(Z_{t+1}=z \mid Z_1=z_1,\hdots,Z_t=z_t) = \frac{e^{((\sigma_t^{(L)})^\tp
    C^\xfer)(z)}}{\sum_{z'\in\bO} e^{((\sigma_t^{(L)})^\tp C^\xfer)(z')}},\quad
t=1,2,\hdots,T,\quad z\in\bO.
\end{equation*}
The operation on the right-hand side is referred to as softmax.

Internally, a transformer is arranged in $L$ layers as follows:
\begin{align*}
&\text{(input)} \quad [z_1,z_2,\hdots,z_T]_{1\times T} \mapsto
                 [e_1,e_2,\hdots,e_T]_{d\times T} \qquad
                 \text{(embedding + positional-encoding)}\\[5pt]
&\text{(first layer)} \quad  [e_1,e_2,\hdots,e_T]_{d\times T} \mapsto [\sigma_1^{(1)},\sigma_2^{(1)},\hdots,\sigma_T^{(1)}]_{d\times T}\\[5pt]
  &\text{(intermediate layer)} \quad  [\sigma_1^{(\ell)},\sigma_2^{(\ell)},\hdots,\sigma_T^{(\ell)}]_{d\times T}  \mapsto  [\sigma_1^{(\ell+1)},\sigma_2^{(\ell+1)},\hdots,\sigma_T^{(\ell+1)}]_{d\times T},\quad \ell=1,2,\hdots,L-1\\[5pt]
  &\text{(output)} \quad
    [\sigma_1^{(L)},\sigma_2^{(L)},\hdots,\sigma_T^{(L)}]_{d\times
    T} \mapsto [p_1,p_2,\hdots,p_T]_{m\times T} \qquad \text{(un-embedding)}
\end{align*}

The correspondence between the dual filter (iterative) and the transformer is as follows:
\begin{enumerate}
\item Input:
  \begin{subequations}
 \begin{align*}
   \text{(dual filter)}\qquad \rho_t^{(0)}(x) &= \frac{C(x,z_t)}{\sum_{x'\in \bS} C(x',z_t)} ,\quad x\in\bS,
   \quad 1\leq t\leq T,
   \\
     \text{(transformer)}\qquad e_t(x) & =C^\xfer(x,z_t) +W_p(x,t) ,\quad x\in\bS,
   \quad 1\leq t\leq T.
 \end{align*}
\end{subequations}
\item Layers:
  \begin{subequations} \label{eq:layer_xfer_dual}
 \begin{align}
   \text{(dual filter)}\qquad
   [\rho_1^{(\ell)},\rho_2^{(\ell)},\hdots,\rho_T^{(\ell)}]_{d\times
   T}  \mapsto
   [\rho_1^{(\ell+1)},\rho_2^{(\ell+1)},\hdots,\rho_T^{(\ell+1)}]_{d\times
   T},\quad \ell=1,2,\hdots,L-1,
   \\
     \text{(transformer)}\qquad [\sigma_1^{(\ell)},\sigma_2^{(\ell)},\hdots,\sigma_T^{(\ell)}]_{d\times T}  \mapsto  [\sigma_1^{(\ell+1)},\sigma_2^{(\ell+1)},\hdots,\sigma_T^{(\ell+1)}]_{d\times T},\quad \ell=1,2,\hdots,L-1.
 \end{align}
 \end{subequations}
\item Output:
 \begin{subequations}\label{eq:p_T_formula}
 \begin{align*}
   \text{(dual filter)}\qquad p_t(z) &= \sum_{x\in\bS} \rho_t^{(L)}(x) C(x,z),\quad
                               t=1,2,\hdots,T, \quad
                                       z\in\bO, 
   \\
     \text{(transformer)}\qquad \ln p_t(z) &=  \sum_{x\in\bS} \sigma_t^{(L)}(x) C^\xfer(x,z) + \text{(constant)},\quad t=1,2,\hdots,T, \quad z\in\bO. 
 \end{align*}
\end{subequations}
\end{enumerate}

\subsection{Layer operation in a transformer}
Because each layer in a transformer is structurally identical, for the
purposes of describing the mapping, it suffices to consider a generic layer whose input
and output are denoted by,
\begin{equation*}
\sigma := [\sigma_1, \sigma_2,\hdots, \sigma_T]_{d\times T} \in  {\cal M}(\bS)^T,\quad \sigma^+:=
[\sigma^+_1, \sigma^+_2,\hdots, \sigma^+_T]_{d\times T} \in {\cal M}(\bS)^T,
\end{equation*}
respectively, and the layer mapping is denoted by
  \begin{equation*}
  {\cal N}^\xfer:{\cal M}(\bS)^T \mapsto {\cal M}(\bS)^T\quad
  \text{such that}\quad \sigma^+ = {\cal N}^\xfer(\sigma).
\end{equation*}
For the first layer, the input is given by the embedding plus
positional encoding.  For each subsequent layer, the input is defined
by the output of the preceding layer.

Operationally, a layer is comprised of multiple attention heads.  A single attention head is an input-output map of the form
  \begin{equation*}
  [\sigma_1, \sigma_2,\hdots, \sigma_T]_{d\times T} \mapsto
  [o_1^h,o_2^h,\hdots,o_T^h]_{d_V \times T}.
  \end{equation*}
  where the output vector $o_t^h \in\Re^{d_V}$ has
  dimension $d_V = \frac{d}{n_{\text{head}}}$, with $n_{\text{head}}$
  denoting the number of heads.
  The operations in a single attention head are defined by three matrices
  \begin{equation*}
  W_V^h \in \Re^{d_V \times d}, \;\;W_Q^h \in \Re^{d_K \times d}, \;\;W_K^h \in \Re^{d_K \times d},
  \end{equation*}
  where typically $d_K=d_V$. The matrices are constant (fixed) during
  the inference phase of the transformer operation.  

The mathematical operation defining the {\em causal self-attention} is as follows:
  \begin{equation}\label{eq:transf_output}
  o_t^h = W_V^h \left( \sum_{s=1}^t \alpha(s;t,h) \sigma_s\right),\quad t=1,2,\hdots, T,
  \end{equation}
  where $\{\alpha(s;t,h):s=1,2,\hdots,t\}$ is referred to as the
  attention weight vector.  For each fixed $(t,h)\in \{1,2,\hdots,T\}\times\{1,2,\hdots,
  {n_{\text{head}}}\}$, it is
  a probability vector.  That is,
\begin{equation*}
  \alpha(s;t,h)\geq 0,\quad s=1,2,\hdots,t,\quad \text{and}\quad 
  \sum_{s=1}^t \alpha(s;t,h)= 1.
  \end{equation*}

  The attention weights are computed as follows:  For each fixed $1\leq t\leq T$, 
  define \begin{align*}
  \text{(query)} \qquad q_t &= W_Q^h \sigma_t, \\
           \text{(key)} \qquad k_s &= W_K^h \sigma_s,\quad s=1,2,\hdots, t.
\end{align*}
Then
\begin{equation*}
\text{(attention weight)} \qquad \alpha(s;t,h) = \frac{\exp({\frac{q_t^\transpose k_s}{\sqrt{d_K}}})}{\sum_{\tau=1}^t \exp({\frac{q_t^\transpose k_\tau}{\sqrt{d_K}}})},\quad s=1,2,\hdots, t.
\end{equation*}

\begin{remark}[Attention as a nonlinear predictor]\label{rem:RPE}
Two remarks on the interpretation of attention as a nonlinear predictor
are as follows:
  \begin{enumerate}
 \item In the terminology adopted in this paper, the dependence of the
attention weights $\alpha_s^t$ upon the data (which varies between sample paths of
observations), makes attention a nonlinear predictor.  If the weights
were deterministic (same for all sample paths), attention will be an
example of a linear predictor (see Remark~\ref{remark:connection_to_transformer}).
\item 
  In the kernelized relative position encoding (RPE),
  \begin{equation*}
  \alpha_s^t = \text{softmax}(\frac{1}{\sqrt{d_K}}q_t^\transpose k_s +
  b(t-s)),\quad s=1,2,\hdots,t,
  \end{equation*}
    where the function $b(\cdot)$ may be pre-specified~\citep{press2021train} or learned~\citep{chi2022kerple}. 
    By itself, without the $q_t^\transpose k_s$ term, the resulting predictor is linear.
  \end{enumerate}
  \end{remark}

  \begin{remark}[Properties of attention
    operation]\label{rem:symmetry} Two salient properties of 
    attention are as follows:
    \begin{enumerate}
    \item Attention operation~\eqref{eq:transf_output} is the \underline{only} mechanism by
    which the input data $e_s$ at other time indices
    ($s=1,2,\hdots,t-1$) influences the output at time $t$.   
All other operations in the transformer are applied independently at
each fixed $t$.
\item 
      Attention operation~\eqref{eq:transf_output} exhibits a symmetry:  
  permuting the inputs $\{e_1,e_2,\hdots,e_{t-1}\}$ (while keeping $e_t$ fixed)  
  yields the identical output $y_t^h$ at time $t$.  
  \end{enumerate}
\end{remark}


From~\eqref{eq:transf_output}, the
layer output is obtained as 
\begin{equation*}
\sigma^+_t = W_O \; \text{concat}(o_t^1, \hdots, o_t^{n_{\text{head}}}),\quad t=1,2,\hdots,T,
\end{equation*}
where $W_O \in \Re^{d\times d}$.

A potential issue during the training phase of a transformer is that the output may `blow up'.  For this reason, several miscellaneous operations are implemented.  All of these operations are
carried out independently for each fixed $t$.  Because these
operations are not especially pertinent to establishing the
correspondence with the dual filter, they are described in
Appendix~\ref{sec:xfer}.

\begin{table}[t]
  \centering
  \begin{tabularx}{\textwidth}{l l l}
    \toprule
    \textbf{Concept} & \textbf{Probabilistic model} & \textbf{Surrogates in a transformer} \\
    \midrule
    State space & \( \bS = \{1,2,\ldots,d\} \) & Dimension $d$ of embedding vector \\
    Observation space & \( \bO = \{0,1,2,\ldots,m\} \) & Vocabulary \\
    Discrete-time & \( \mathbb{T} = \{0,1,2,\ldots,T\} \) & $T$ is context length \\
    Observation process & \( Z= [Z_1, Z_2, \ldots, Z_T] \;\text{is}\; \bO\text{-valued}\) & Token-ized prompt \\
    Latent process & \( X = [X_0, X_1, \ldots, X_T]\;\text{is}\; \bS\text{-valued}\) & Internal latent state in language\\
    Observation model & \( C(x,z)=\sP(Z_{t+1} = z \mid X_t = x),x
                        \in\bS,z\in \bO\) & Embedding matrix \( C^\xfer(x, z) \) \\
    \bottomrule
  \end{tabularx}
  \vspace*{-10pt}
  \caption{Probabilistic model and its surrogates in a transformer.}
  \label{tab:correspondence_model}
\end{table}

\begin{table}[h]
  \centering
  \begin{tabularx}{\textwidth}{l l l}
    \toprule
    \textbf{Concept} & \textbf{Inference architecture} & \textbf{Surrogates in a transformer} \\
    \midrule
    Layer operation & \\
    Layer mapping & $\clN^\dfltr:\clP(\bS)^T \mapsto \clP(\bS)^T $ & ${\cal N}^\xfer:{\cal M}(\bS)^T \mapsto {\cal M}(\bS)^T$ \\
    Layer $\ell$ input & \( \rho^{(\ell)} \) & \( \sigma^{(\ell)}\) \\
    Layer $\ell$ output & \( \rho^{(\ell+1)} =
                          \clN^\dfltr(\rho^{(\ell)}) \) & \(
                                                          \sigma^{(\ell+1)}
                                                          = {\cal
                                                          N}^\xfer(\sigma^{(\ell)})
                                                          \) \\
     \midrule
    Final layer output & \\
    Cond. measure & $\pi=[\pi_1,\pi_2,\hdots,\pi_T]
                          \;\text{is}\; \clP(\bS)\text{-valued}$ & $\sigma^L=[\sigma^L_1,\sigma^L_2,\hdots,\sigma^L_T]
                          \;\text{is}\; \clM(\bS)\text{-valued}$ \\
    Fixed-point & \( \pi = \clN^\dfltr(\pi) \) & unknown \\
    Prediction & \( p_t(z)= \sum_{x}\pi_t(x)C(x,z),\;z\in\bO\) & \( \ln p_t(z) = \sum_{x}\sigma_t^{(L)} (x)C^\xfer(x,z)+\text{(const)}\) \\
    \bottomrule
  \end{tabularx}
  \vspace*{-10pt}
  \caption{Inference architecture and its surrogates in a transformer.}
  \label{tab:correspondence_solution}
\end{table}

\subsection{Correspondence between the layer operations}
\label{sec:correspondence_layer_ops}

The layer operations in the dual filter and the transformer can be
viewed abstractly as maps between sequences of
measures. \Fig{fig:transformer_objective} illustrates this
comparison for the two maps,
\begin{align*}
  \text{(dual filter)}\qquad {\cal N}^\dfltr&:\clP(\bS)^T \mapsto \clP(\bS)^T,\\
  \text{(transformer)}\qquad
{\cal N}^\xfer&:{\cal M}(\bS)^T \mapsto {\cal M}(\bS)^T.
\end{align*}
There are two structural similarities between these maps:
\begin{enumerate}
\item For both of these maps, the input and also the output at time $t$ is a $d\times 1$ vector, for $t=1,2,\hdots,T$.
\item The final layer output at time $t$ is used to approximate the prediction $p_t$ at time $t$.
\end{enumerate}
For the dual filter, these properties follow from the probabilistic modeling which is closely inspired by the surrogates in the transformer (see Table~\ref{tab:correspondence_model}).  
Note that there is no surrogate in the transformer for the state transition matrix $A$ which is used here to model the latent process $X$ as a Markov chain.

\begin{figure}[t!]
  \centering
  \input{figure/transformer_layer.tex}
  \vspace{-10pt}
  \caption{
Structural correspondence between the dual filter layer and a
self-attention layer in a decoder-only transformer. 
(left) The proposed layer map
$\clN^\dfltr:\rho \in \clP(\bS)^T \mapsto \rho^+ \in \clP(\bS)^T$.
(right) A transformer layer
$\clN^\xfer:\sigma \in \clM(\bS)^T \mapsto \sigma^+ \in \clM(\bS)^T$.
In both cases, the transformation maps a sequence of $d$-dimensional
vectors to another sequence of $d$-dimensional vectors of the same
length.
}
  \label{fig:transformer_objective}
  \vspace*{-20pt}
\end{figure}

It is important to distinguish between the $\bS$-valued latent
process $X$ and the Markov assumption imposed on it. The
$\clP(\bS)$-valued conditional measure $\pi$ and the nonlinear
predictor~\eqref{eq:nonlin_predictor_rep} are well-defined even
without assuming that $X$ is Markov (see
Remark~\ref{rem:non_Markovian} in
Appendix~\ref{proof:prop:existence_nonlin_predictor_rep}). The
correspondence in Table~\ref{tab:correspondence_solution} emphasizes
the layer mapping and its role in computing $\pi$ as a fixed point.

The Markov assumption is not a modeling claim about transformers. Its
sole purpose here is analytic: it enables an explicit expression for
the control $U$ in~\eqref{eq:nonlin_predictor_rep}, and hence an
explicit form of the layer map ${\cal N}^\dfltr$. In this sense, the
Markov model serves as a tractable baseline from which one can
investigate principled modifications of the transition structure
(see~\cite{alaa2019attentive,tang2021probabilistic}) that move the
architecture closer to modern attention-based models.

\subsection{Comparison with prior modeling work}
\label{rem:IPS}

At the core of the dual filter is the mapping $\clN^\dfltr:\clP(\bS)^T\to \clP(\bS)^T$ which
defines a transport on the space of probability measures.  This
is contrasted with the mathematical modeling of a transformer described
in~\citep{geshkovski2023mathematical,geshkovski2024measure,abella2024asymptotic,adu2024approximate,castin2025unified}. 
For this purpose, it is useful to first recall the input-output operation of
a single layer (repeated from~\eqref{eq:layer_xfer_dual} above):
\begin{align*}
   \text{(dual filter)}\qquad
  [\rho_1^{(\ell)},\rho_2^{(\ell)},\hdots,\rho_T^{(\ell)}]_{d\times T}
  \mapsto  [\rho_1^{(\ell+1)},\rho_2^{(\ell+1)},\hdots,\rho_T^{(\ell+1)}]_{d\times T},\quad \ell=1,2,\hdots,L-1,
   \\
     \text{(transformer)}\qquad
  [\sigma_1^{(\ell)},\sigma_2^{(\ell)},\hdots,\sigma_T^{(\ell)}]_{d\times
  T}
  \mapsto  [\sigma_1^{(\ell+1)},\sigma_2^{(\ell+1)},\hdots,\sigma_T^{(\ell+1)}]_{d\times T},\quad \ell=1,2,\hdots,L-1.
\end{align*}
The viewpoint espoused in these earlier studies is to regard 
$\sigma^\ell=\{\sigma_t^{(\ell)}:1\leq t\leq T\}$ as an ensemble.
These are suitably normalized such that each ensemble member
(particle) $\sigma_t^{(\ell)}$ is an element of the sphere 
$S^{d-1}:=\{s\in\Re^d:|s|=1\}$ for $t=1,2,\hdots,T$, and the
ensemble defines an empirical measure in $\clP(S^{d-1})$.  The
objective of the prior work is to study the nonlinear mapping
$\sigma^\ell \mapsto \sigma^{\ell+1}$ as a transport on
$\clP(S^{d-1})$.  For tractability, several simplifying
assumptions are necessary. These include assuming an exchangeability
property of the ensemble members (this property does not hold with
causal masking), ignoring the effect of positional encoding, and
taking a continuous approximation of the discrete layers in order to
derive a continuity equation for the transport. 

Summarizing, both our work and the prior studies provide mathematical
frameworks to model and understand a transformer as a transport on an
appropriate 
space of probability measures. Notably, the spaces are very different,
$\clP(\bS)^T$ for us and $\clP(S^{d-1})$ for the prior work. 
Moreover, the approaches differ in
focus.  The
prior work, because it explicitly models the attention mechanism as
it is implemented, is more faithful to the actual architecture of the
transformer. In contrast, our work models the functional goal of
the transformer---namely to predict the next token---rather than its
detailed implementation.

\section{Numerics with the dual filter}
\label{sec:numerics}

The numerical experiments are designed to validate the dual filter on
a known HMM, where the parameters are chosen to mimic certain models
of transformers.  Specifically, we adopt parameter
settings following the character-level configuration from Karpathy's nanoGPT
implementation~\citep{Karpathy2022}: $d = 384$, $m = 65$, and $T =
256$. 
The HMM parameters $(\mu, A, C)$ are configured as follows:
\begin{itemize}
  \item The prior $\mu$ is set to the uniform probability vector: $
\mu(x) = \frac{1}{d}$ for $x\in\bS$.
  \item The transition matrix $A$ is a convex
    combination of two components: a cyclic permutation matrix 
    $A^\text{(circ)}$ (encoding a deterministic transition $0\mapsto 1 \mapsto 2 \mapsto
    \hdots \mapsto (d-1) \mapsto 0$) and a randomly sampled 
    stochastic matrix $A^\text{(stoch)}$,
  \begin{equation*}
    A = \alpha A^\text{(circ)} + (1-\alpha) A^\text{(stoch)},\quad \text{where} \quad
    A^\text{(circ)}(x,x') = \begin{cases}
      1, & x'= x+1 \mod d \\
      0, & \text{o.w.}
    \end{cases},\quad x,x'\in\bS
  \end{equation*}
  where $\alpha\in (0,1)$ is a homotopy parameter. 
\item The emission matrix $C$ is randomly sampled. See Algorithm~\ref{alg:sparse_stochastic} for details.
  \item In both $A^{\text{(stoch)}}$ and $C$, each row is sampled
    independently: row-entries are drawn i.i.d. from a standard Normal
    distribution and then normalized using the softmax operation to
    form a valid probability vector. 
  \end{itemize}
Spectral properties of the transition matrix $A$ as a function of the
homotopy parameter $\alpha$ are illustrated
in~\Fig{fig:eigenvalue_results}.

\begin{figure}[t]
  \centering
  \includegraphics[width=\textwidth]{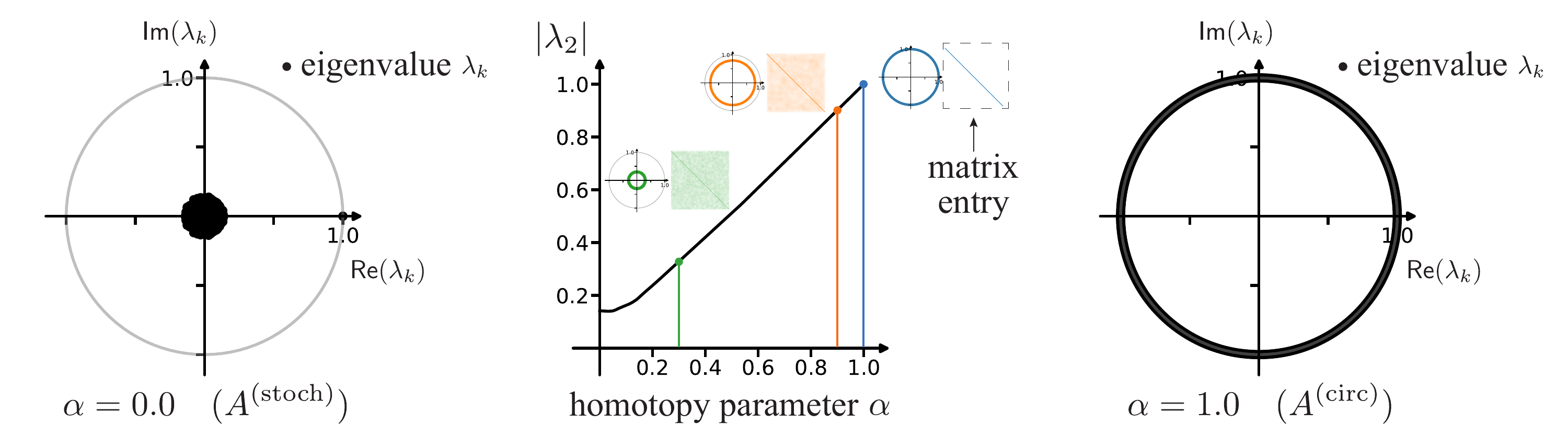}
  \caption{
Eigenvalues of the transition matrix $A$ as a function of the homotopy parameter $\alpha$.
(Middle): Plot of the second eigenvalue magnitude $|\lambda_2|$ as a
function of $\alpha$.
(Left): Eigenvalue spectrum for $A = A^{\text{(stoch)}}$ ($\alpha = 0$).
(Right): Eigenvalue spectrum for $A = A^{\text{(circ)}}$ ($\alpha = 1$).
In the middle plot, three representative values of $\alpha$---0.3 (green), 0.9 (orange), and 1.0 (blue)---are highlighted.
Insets show the corresponding eigenvalue spectra and matrix
structures; darker shades indicate higher matrix entry values.  
  }
  \label{fig:eigenvalue_results}
\end{figure}

\subsection{Dual filter algorithm}

An observation sequence $\{z_1,z_2,\hdots,z_T\}$ is a single sample path generated from the HMM defined by the parameters above with $\alpha=1$. 
The ground truth is obtained by simulating the nonlinear filter (forward algorithm) to compute the conditional probability $p_t(z)$ at each time step $t \in \mathbb{T}$ and $z\in\bO$.  
This is denoted by $p_t^{\text{(forward)}}(z)$ and compared with the conditional probability computed using the two dual filter algorithms:
\begin{enumerate}
\item \Fig{fig:one_shot_algorithm_results} depicts the comparison with
  the single-shot algorithm.
\item \Fig{fig:iterative_algorithm_results} depicts the comparison with
  the iterative algorithm. 
\end{enumerate}
The following error metric is used to help illustrate the convergence,
\begin{equation*}
  \varepsilon_t^{(\ell)}\coloneqq\varepsilon(p_t^{(\ell)};p_t^\text{(forward)}),\quad 
  \varepsilon(p';p)\coloneqq \max_{z\in\bO} \left\lvert
    p'(z)-p(z)\right\rvert,\quad \ell=1,2,\hdots,L.
\end{equation*}
For the single-shot algorithm, only a single iteration is necessary,
and therefore, $L=1$.  In what follows, all computations
are carried out with the single-shot algorithm.  

\begin{figure}[t]
  \centering
  \includegraphics[width=\textwidth]{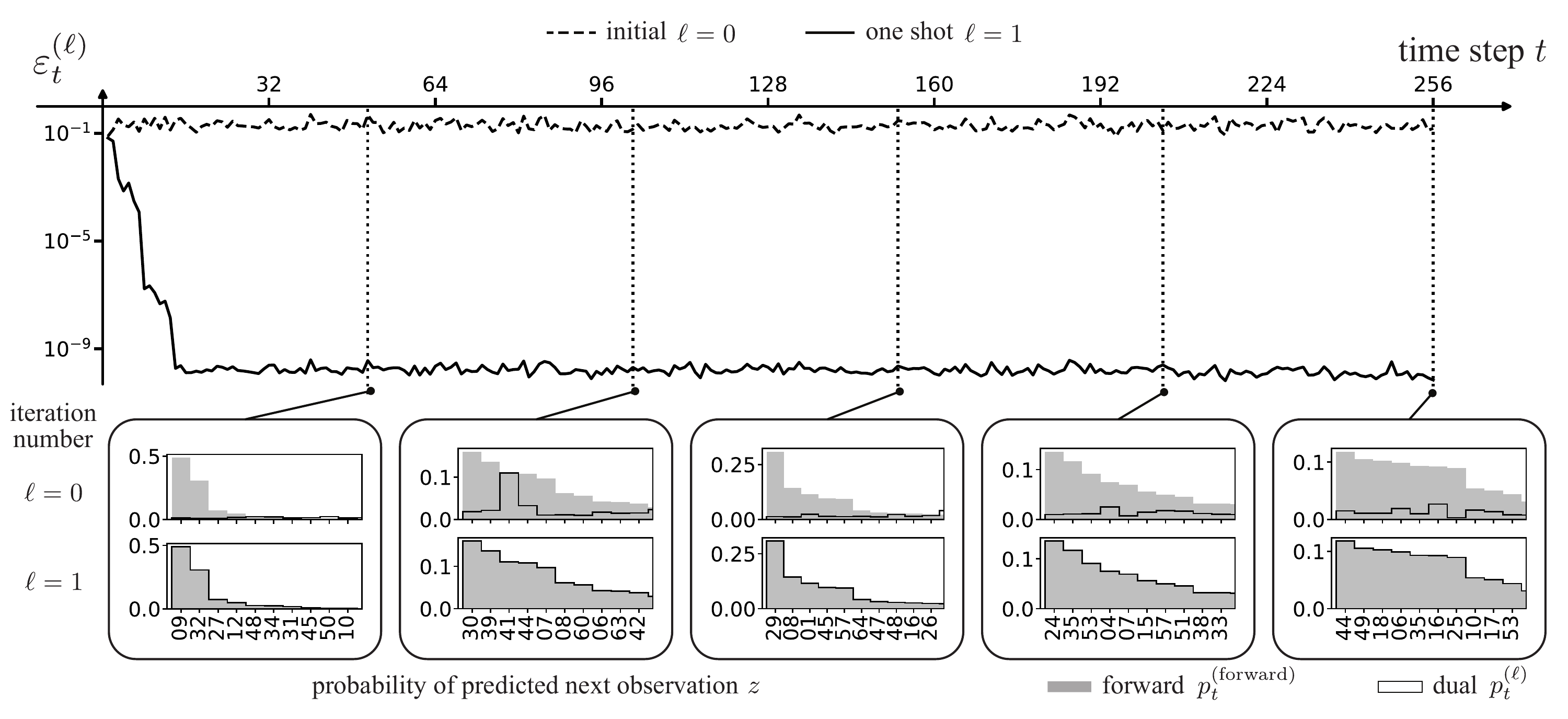}
  \vspace*{-25pt}
  \caption{
Comparison with the single-shot algorithm.
(Top): Time traces of the error $\{\varepsilon_t^{(\ell)} : 1 \leq t
  \leq T\}$ for $\ell = 0, 1$. The dashed line corresponds to the
initial error ($\ell = 0$), while the solid line shows the error after
one iteration ($\ell = 1$). 
(Bottom): Top-ten conditional probabilities $\{p_t^{(\ell)}(z_i) : i = 1,
  \ldots, 10\}$ for five representative time points. For each $t$,
  these are obtained by sorting the conditional probability vector and
  selecting the top ten.   The gray shading indicates the
ground truth (from the nonlinear filter), and the solid line shows the
result from the dual filter.   
  }
  \label{fig:one_shot_algorithm_results}
  \vspace*{-20pt}
\end{figure}

\begin{figure}[t]
  \centering
  \includegraphics[width=\textwidth]{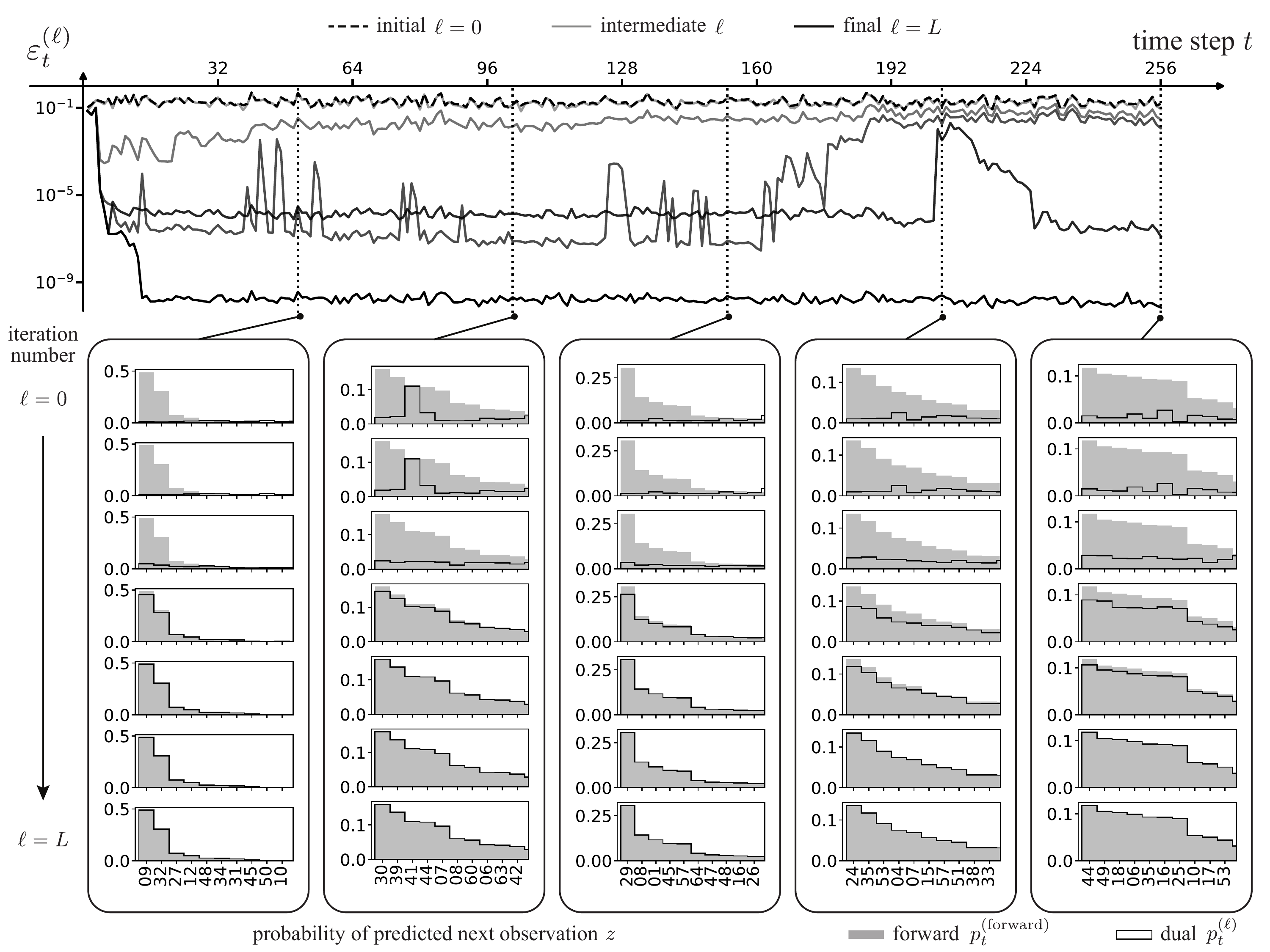}
  \vspace*{-25pt}
  \caption{
    Comparison with the iterative algorithm: (Top): Time traces of the
    error $\{\varepsilon_t^{(\ell)} : 1 \leq t \leq T\}$ for  
    $\ell=1,2,\hdots,L$ (with $L=6$).  The dashed line corresponds to the
initial error ($\ell = 0$), while the solid lines show the error after 
    subsequent iterations, with darker shades used as $\ell$
    increases. 
    (Bottom): Top-ten conditional probabilities $\{p_t^{(\ell)}(z_i) : i = 1,
  \ldots, 10\}$ for five representative time points. For each $t$,
  these are obtained by sorting the conditional probability vector and
  selecting the top ten.  The gray shading indicates the
ground truth (from the nonlinear filter), and the solid line shows the
result from the dual filter.  
  }
  \label{fig:iterative_algorithm_results}
  \vspace*{-25pt}
\end{figure}

\subsection{Optimal control input}

\begin{figure}[t]
  \centering
  \includegraphics[width=\textwidth]{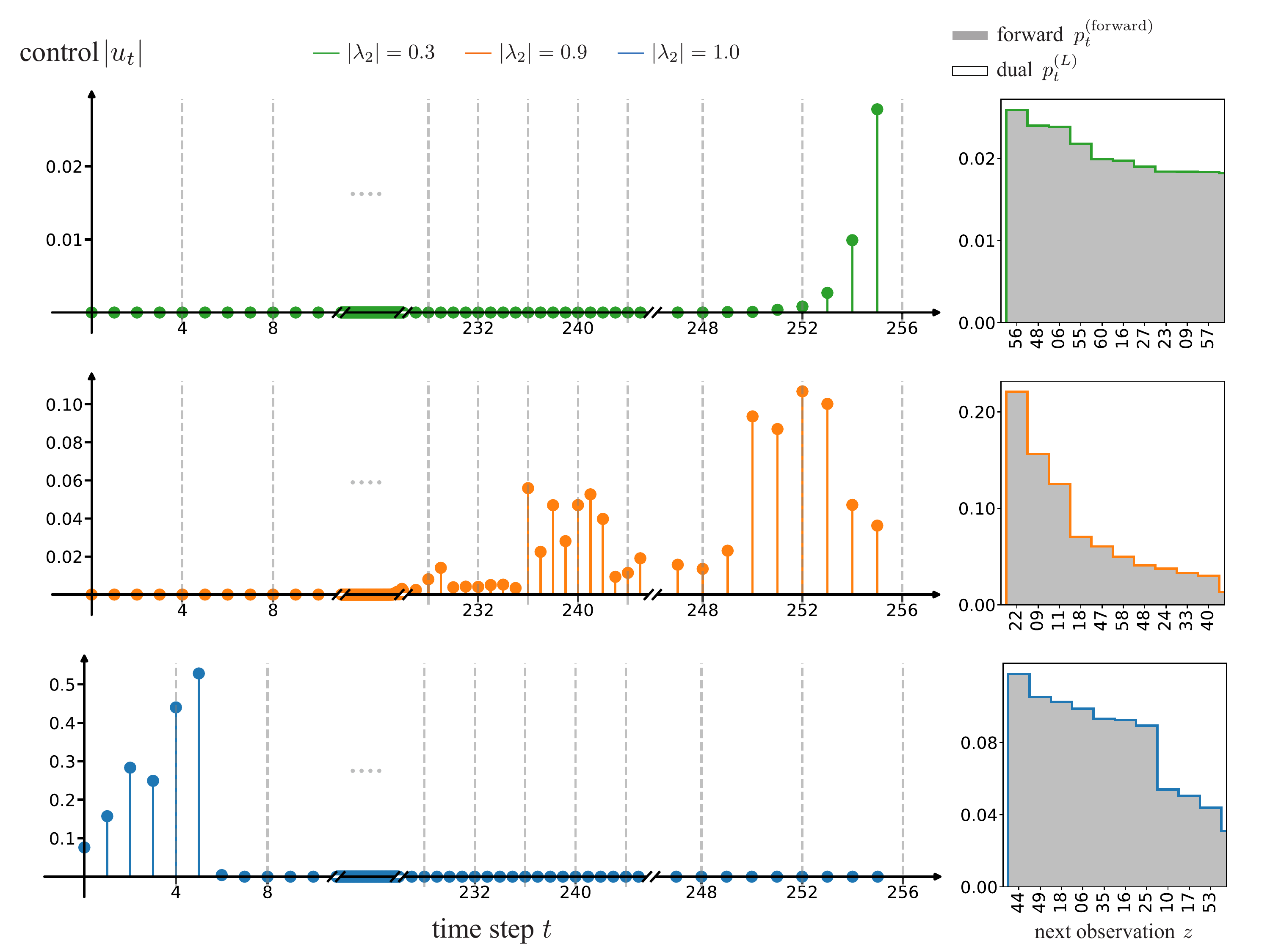}
  \vspace*{-25pt}
  \caption{
    Optimal control inputs as a function of $|\lambda_2|$. 
    (Top): $|\lambda_2| = 0.3$ (green),
    (Middle): $|\lambda_2| = 0.9$ (orange),
    (Bottom): $|\lambda_2| = 1.0$ (blue).
    These three cases correspond to different values of the homotopy parameter $\alpha$. 
    Note that the x-axis (time $t$) is plotted on a non-linear scale.
    The right-hand panels show the top ten conditional probabilities
    $\{p_T^{(\ell)}(z_i): i = 1, \ldots, 10\}$ at the terminal time $t =
    T$. 
    These are obtained by sorting the conditional probability vector and selecting the top ten. 
  }
  \label{fig:control_results}
  \vspace*{-25pt}
\end{figure}

\begin{figure}[t]
  \centering
  \includegraphics[width=\textwidth]{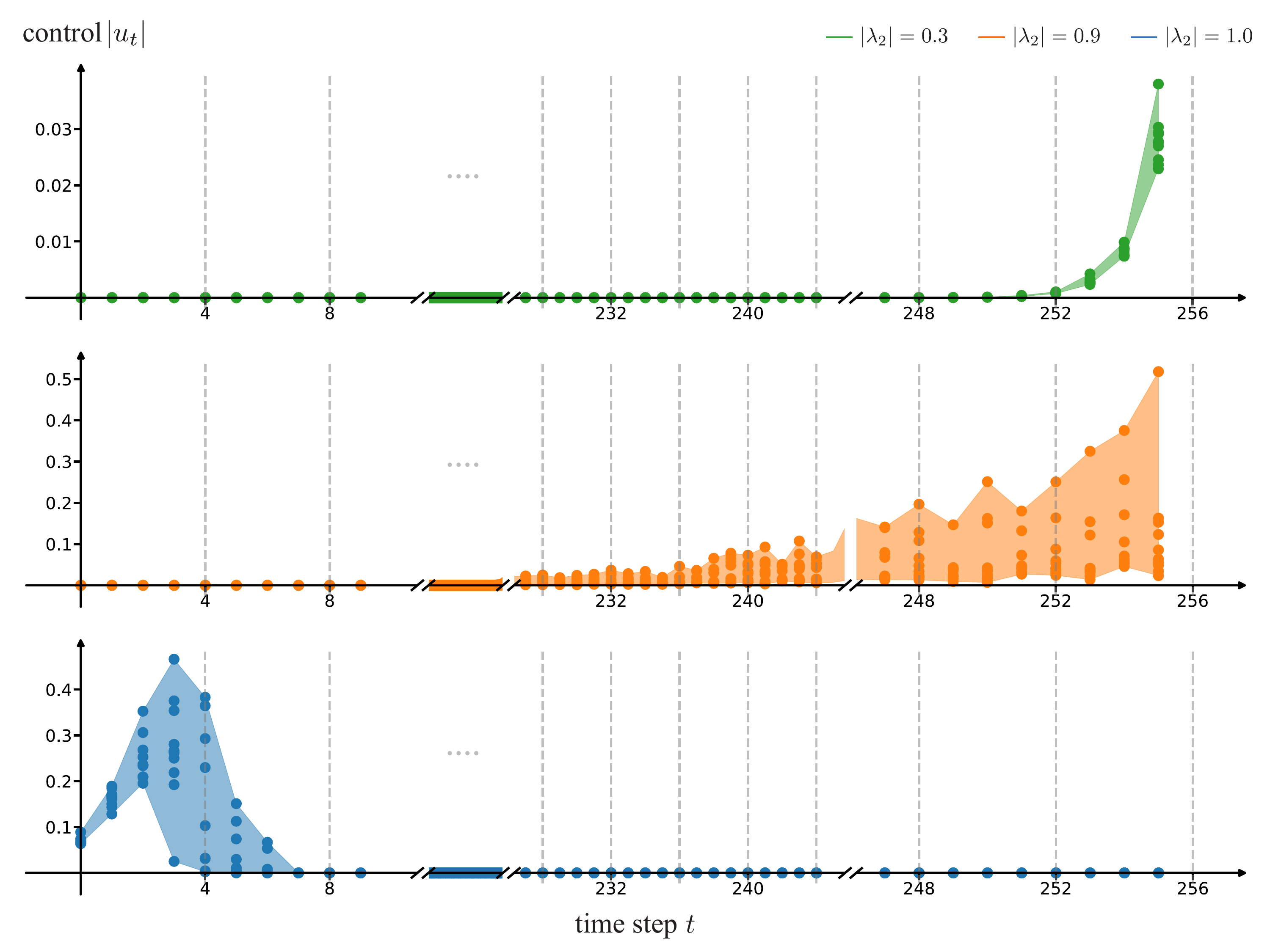}
  \vspace*{-25pt}
  \caption{
    Optimal control inputs across 10 sample paths: 
    (Top): $|\lambda_2| = 0.3$ (green),
    (Middle): $|\lambda_2| = 0.9$ (orange),
    (Bottom): $|\lambda_2| = 1.0$ (blue).  
    The shaded region shows the range of control values (minimum to maximum) across the 10 realizations. 
    Note that the scale for the $x$-axis (denoting time $t$) is not linear.  
  }
  \label{fig:controls_results}
  \vspace*{-25pt}
\end{figure}

Representation~\eqref{eq:nonlin_predictor_rep} is helpful for visualizing and understanding the short- and long-term correlations.
To illustrate this, \Fig{fig:control_results} shows the optimal control input for three different values of the homotopy parameter $\alpha$:
\begin{enumerate}
  \item For a small value of $\alpha$ such that $|\lambda_2| = 0.3$, the Markovian dynamics mix rapidly. 
  As a result, only the most recent observations contribute meaningfully to prediction at the terminal time $T$. 
  The control is nonzero only over the most recent few time steps.

  \item For $\alpha = 1$, the dynamics are deterministic. 
  In this case, the control is nonzero during the initial phase ($t \leq 5$) when the state is uncertain. 
  At $t = 6$, after the first six observations, the state becomes known (up to numerical precision), and the control input drops to zero (again, up to numerical precision).
  
  \item For an intermediate value of $\alpha$ such that $|\lambda_2| = 0.9$, the control input remains nonzero over an extended time horizon, reflecting the presence of longer-range correlations.
  Among the three cases, we believe this setting is most representative of the behavior observed in transformer-based predictions. 
\end{enumerate}

While \Fig{fig:control_results} shows the control input for a single sample path of the output, \Fig{fig:controls_results} depicts control inputs computed for ten independent sample paths. 
For each fixed value of $\alpha$, ten sample paths of the output sequence are generated. 
For each sample path, the single-shot dual filter is simulated to compute the corresponding control input.

For $|\lambda_2|=0.3$, the control inputs are nearly identical across all ten sample paths, suggesting that the predictor behaves approximately linearly.
In contrast, for the other two cases, where the spectral gap is smaller, the control inputs exhibit greater variability.  
Nevertheless, the qualitative structure of the control---such as the time window over which it is nonzero---is consistent across sample paths. This consistency is because of the relatively simple structure of the transition matrix.  
For a more complicated multi-scale-type choice of the transition matrix $A$, this is expected to change.   

\subsection{Comparison with transformer}
\label{sec:comparison_transformer}

While the scope of numerics is limited to the
validation of the dual filter algorithm in model-based settings, a
small set of numerical experiments were also carried out using
nanoGPT~\citep{karpathy2024nanogpt}.
For these experiments, nanoGPT is first trained using the data generated
from an HMM.  Once trained, the model is evaluated for the next-token
prediction task.  A summary of these numerical evaluations is as
follows: 
\begin{enumerate}
\item For model choices where the prediction $p_T$ is supported on a
  relatively small subset of the vocabulary $\bO$, the transformer
  exhibits good performance comparable to the dual filter.
\item For other choices, including for the model discussed as part of
  \Fig{fig:iterative_algorithm_results}, the nanoGPT performance is
  comparatively poor.
\end{enumerate}
The details are included in Appendix~\ref{sec:app_transformer}.  

A direct empirical comparison of the dual filter (or a learning
variant) against transformer models on general ML tasks (like language
modeling) requires the development of a learning framework that is
outside the scope of the current foundational paper.   Some of these
learning-related tasks are planned as part of the future work and
briefly described next.

\section{Discussion and directions for future work}
\label{sec:conc}

In this paper, we developed theory and algorithms for a nonlinear
predictor based on the representation
in~\eqref{eq:nonlin_predictor_rep}.  We identify two uses of this
representation, which motivate two main directions for future research.

\newP{1. Analysis} A key outcome of this paper is the fixed-point
representation of the conditional measure (posterior) $\pi$.  The
representation invites work on contraction analysis of the mapping 
$\clN^\dfltr$ as well as for the analysis of filter stability. 


\newP{2. Learning} The objective here is to develop algorithms to
learn the layer transformation $\clN^\dfltr$.  In its basic form, this
may be accomplished through learning a parametrized model for
$(\mu,A,C)$.      
There are lessons to be drawn from the reinforcement learning (RL) literature in the setting of partially observed Markov decision processes (POMDPs).  
In such problems, a key challenge is that the model $(\mu,A,C)$ of the hidden Markov process is not known, making it difficult to represent and compute the nonlinear filter (or belief state) directly.
To solve this problem, two complementary approaches are commonly considered:
\begin{enumerate}
  \item A generative model approach, where a parametrized model (for $(\mu,A,C)$) is learned and subsequently used to construct the filter (e.g., using the forward algorithm described in \Sec{sec:intro}).
  \item A history-based approach, where representations are learned directly based on the history---past observations and past actions in the settings of POMDP. 
\end{enumerate}

Since the 1990s, empirical evidence suggests that the history-based approach is preferable in RL settings~\citep{mccallum1996reinforcement}.  
This has led to the development of several frameworks, namely, predictive state representation~\citep{littman2001predictive,rudary2003nonlinear}, which is itself based on the observable operator models (OOM) in un-controlled settings~\citep{jaeger2000observable}, and more recently the approximate information state (AIS)~\citep{subramanian2022approximate,kao2022common}.
It will be of interest to relate~\eqref{eq:nonlin_predictor_rep} to
these history-based approaches. 
The theoretical framework presented in this paper
provides a baseline for what an inference architecture can achieve when
the data is generated from an HMM. Of course, it will be interesting
to extend the framework to a more general class of non-Markovian
correlations found in the real-world data. A more fruitful avenue may
be to understand the significance of differences, specifically,
\begin{itemize}
\item the use of positional encoding to encode the temporal structure
  in a transformer; and
\item the advantages such encodings provide for learning and
  prediction tasks. 
\end{itemize}
A related goal is the use of the representation~\eqref{eq:nonlin_predictor_rep} for POMDPs. Empirical studies that use transformer architectures for control have appeared in~\cite{chen2021decision,zhang2024decision,guo2024controlagent,ziemann2024state}.

For both analysis and learning, the recent paper of~\cite{yuksel2025another} is foundational.  
The paper describes conditions---related to filter stability~\citep{van2009observability,chigansky2009intrinsic}---under which purely data-driven approximations are meaningful for control.
It is a goal of the ongoing research to better understand the role that these conditions play for the dual filter.

\pagebreak
  
\acks{%
  This work is supported in part by the AFOSR award FA9550-23-1-0060 and the NSF award 2336137.  
  Prashant Mehta acknowledges several useful conversations on the topics of transformers and large language models with Dr.~Alberto Speranzon. 
  The original research reported in this paper builds upon prior work on duality theory carried out in collaboration with Dr. Jin-Won Kim. 
  These past collaborations are gratefully acknowledged.
}

\appendix

\section{Existence proofs}
\label{appdx:existence}

The existence theorems rely on the following proposition from
linear algebra.

\begin{proposition}\label{prop:linear_algebra}
  Let $s:\bO\to \Re$. Then there exists unique $(s,\tilde{s}) \in \Re
  \times \Re^m$ such that the following decomposition holds:
  \begin{equation*}
  s(z) = \bar{s} +\tilde{s}^\tp e(z) ,\quad z\in\bO.
\end{equation*}
Explicitly, 
\begin{equation*}
\bar{s}:= \frac{1}{m+1} \sum_{z\in\bO} s(z),\quad \text{and} \quad 
\tilde{s} (i) = (s(i) - \bar{s}),\quad
i=1,2,\hdots,m.
\end{equation*}
\end{proposition}
\begin{proof}
  We have 
  \begin{equation*}
\left( \bar{s} +\tilde{s}^\tp e(z) \right)   = \begin{cases}
\left( \bar{s} + ( s(z) -\bar{s} )\right) = s(z) , &  z=1,2,\hdots,m,\\[5pt]
   \left( \bar{s} - \tilde{s}^\tp \ones \right) = s(0),
     & z=0,
     \end{cases}
 \end{equation*}
 where the last step follows because
 \begin{equation*}
   \tilde{s}^\tp \ones = \sum_{i=1}^m \tilde{s}(i)  = \sum_{i=1}^m
     \left(s(i) - \bar{s} \right) = - m \bar{s}  + \sum_{i=1}^m s(i),\end{equation*}
and therefore,  \begin{equation*}
\bar{s} - \tilde{s}^\tp \ones = \bar{s} - \left(- m \bar{s} +
  \sum_{i=1}^m s(i) \right) = (m+1) \bar{s} - \sum_{i=1}^m s(i) = s(0).
\end{equation*}
The decomposition is unique because $\bar{s} + \tilde{s}^\tp e(z)
\equiv 0$ implies
\begin{align*}
  \bar{s} + \tilde{s}^\tp e(i) = \bar{s} + \tilde{s}(i) &= 0, \quad i=1,2,\hdots,m, \\
  \bar{s} + \tilde{s}^\tp e(0) = \bar{s} - \sum_{i=1}^m \tilde{s}(i) &= 0.
\end{align*}
Summing the first of these equations over $i$,
\begin{equation*}
m\bar{s} + \sum_{i=1}^m \tilde{s}(i) = (m+1) \bar{s} =0,
\end{equation*}
which then also implies $\tilde{s}(i) = -\bar{s} =0$ for
$i=1,2,\hdots,m$.         
\end{proof}

\begin{example}[m=1] 
  Let $s:\{0,1\} \to \Re$.  
  Denote $s^+ := s(1)$ and $s^{-} := s(0)$. 
  Then
  \begin{equation*}
    s(z) = \bar{s} +  \tilde{s} \, e(z),\quad z\in\{0,1\},
  \end{equation*}
  where $\bar{s}:=\half (s^+ + s^-)$ and $\tilde{s}:= \half (s^+ - s^-)$ (recall $e(1)=1$ and $e(0)=-1$).
\end{example}

\begin{remark}
  There is nothing special about the choice of $\{e(1),e(2),\hdots,e(m)\}$, chosen in this paper to be the canonical basis.
  One could instead choose these vectors to be any basis of $\Re^m$, and set $e(0)=-e(1)-e(2)-\hdots-e(m)$ as before.
  In~\cite{fukasawa2025backward}, such a structure is referred to as a lattice. 
  See~\citep{cohen2010general} for a general theory of BS$\Delta$E.
\end{remark}

\subsection{Well-posedness of
  representation~\eqref{eq:nonlin_predictor_rep} (Proof
  of~\Prop{prop:existence_nonlin_predictor_rep})}
\label{proof:prop:existence_nonlin_predictor_rep}

\begin{proof}[of \Prop{prop:existence_nonlin_predictor_rep}]
We begin by proving a result where the representation is shown to hold
for {\em any} $S_T \in \clZ_T$.  From Doob-Dynkin lemma, there is a
deterministic function $s:\bO^T \to \Re$ such that
\begin{equation*}
S_T = s(Z_1,\hdots,Z_{T-1},Z_T).
\end{equation*}
Set
\begin{equation*}
S(z) := s(Z_1,\hdots,Z_{T-1},z),\quad z\in\bO.
\end{equation*}
From~\Prop{prop:linear_algebra},
\begin{equation*}
S_T = S(Z_T) = S_{T-1} - (U_{T-1})^\tp e(Z_T),
\end{equation*}
where
\begin{equation*}
  S_{T-1} = \frac{1}{m+1} \sum_{z\in \bO} S(z),\quad 
 U_{T-1}(i):= - (S(i) - S_{T-1}),\quad
i=1,2,\hdots,m.
\end{equation*}
Uniqueness is from the uniqueness of the decomposition. The proof is completed through induction by
repeating the procedure for $S_{T-1}\in\clZ_{T-1}$.

A direct application of the above result to justify the
representation~\eqref{eq:nonlin_predictor_rep} for $\pi_T(F)$ is
complicated by a subtle issue: The conditional expectation $\pi_T(F)$
is meaningfully defined only for sample paths $Z=z$ with 
$\sP([Z=z])>0$.  Note here that because $|\bO|=m+1$ and $T$ are
both finite, there are only finitely many---specifically $(m+1)^T$---sample paths.
Thus, $\sP([Z=z])$ is a well-defined object for each sample path, 
although it may be zero depending on the HMM parameters $(\mu,A,C)$.

There are two ways to address this issue: 
\begin{enumerate}
\item Assume $\underline{c}:=\min\{C(x,z)
  \,:\,x\in\bS,\,z\in\bO\} >0$.  Then
  $\sP([Z=z])\geq\underline{c}^T > 0$ for all $z\in\bO^T$,
and the existence of a unique $U$ follows directly from the earlier result.
\item Adopt the convention $\frac{0}{0}=0$ to define (or extend) the conditional measure for sample
paths $Z=z$ with $\sP([Z=z])=0$. Then again, a particular
selection of $U$ follows from the above result.
\end{enumerate}
In the second case, however, there
may be other choices of $U$ such that the
representation~\eqref{eq:nonlin_predictor_rep} holds: Any two choices
will yield a representation that coincides on the set $\{z\in \bO^{T}:
\sP(Z=z)>0\}$ but may differ on the set $\{z\in \bO^{T}:
\sP(Z=z)=0\}$.
\end{proof}


\begin{remark}[Extension to non-Markovian settings]\label{rem:non_Markovian}
  To prove~\Prop{prop:existence_nonlin_predictor_rep}, we did not make
  use of 
  the Markov property of the latent process $X$.  While the focus of
  the present paper is on the HMM, the
  representation~\eqref{eq:nonlin_predictor_rep} holds also in more
  general non-Markovian settings.  The Markovian assumption is
  introduced solely to obtain an explicit formula for $U$.
\end{remark}

\begin{example}[m=1] 
Set $
S_T^{+} = s(Z_1,\hdots,Z_{T-1},1)$ and $S_T^{-} =
s(Z_1,\hdots,Z_{T-1},0)$. 
Then $S_T^{+}, S_T^{-}\in \clZ_{T-1}$ and
\begin{equation*}
S_T = S_{T-1} - U_{T-1} e(Z_T),
\end{equation*}
where $
S_{T-1}=\half\left( S_T^{+} + S_{T}^-\right)$ and $U_{T-1}= -\half\left( S_T^{+} - S_{T}^-\right)$. 
\end{example}

\subsection{Well-posedness of
  BS$\Delta$E~\eqref{eq:dual_BSDE} (Proof of~\Prop{prop:existence_BSDE})}
\label{proof:prop:existence_BSDE}

\begin{proof}[of \Prop{prop:existence_BSDE}]
 Fix $x\in \bS$.  Then at time $t=T$, the BS$\Delta$E is
  \begin{equation*}
  (AF)(x) = Y_{T-1}(x) - c^\tp(x) (U_{T-1} + V_{T-1}(x)) + (V_{T-1}(x))^\tp
  e(Z_T),\quad x\in \bS.
  \end{equation*}
  Because $A$ is deterministic $(AF)(x)\in\clZ_T$ and therefore there is a
deterministic function $s: \bO^T\times \bS \to \Re$ such that
\begin{equation*}
  (AF)(x) = s(Z_1,\hdots,Z_{T-1},Z_T;x),\quad x\in\bS.
\end{equation*}
Set
\begin{equation*}
S(z;x) := s(Z_1,\hdots,Z_{T-1},z \;\; ; x),\quad z\in\bO,\quad x\in\bS.
\end{equation*}
From~\Prop{prop:linear_algebra}, there exists unique
$\bar{S}_{T-1}(x), \tilde{S}_{T-1}(x) \in \clZ_{T-1}$ such that
\begin{equation*}
(AF)(x) = S(Z_T;x) = \bar{S}_{T-1}(x) + (\tilde{S}_{T-1}(x))^\tp
e(Z_T),\quad x\in \bS.
\end{equation*}
Set
\begin{align*}
  V_{T-1}(x) &=  \tilde{S}_{T-1}(x),\quad x\in\bS,\\
  Y_{T-1}(x) &=\bar{S}_{T-1}(x)  + c^\tp(x) (U_{T-1} + V_{T-1}(x)) ,\quad x\in\bS.
\end{align*}
Uniqueness follows from the uniqueness of decomposition.  Because
$Y_{T-1}\in\clZ_{T-1}$, the proof is completed through induction.  
\end{proof}

\begin{example}[m=1]
For each $x\in\bS$, set $
S_T^{+}(x) = s(Z_1,\hdots,Z_{T-1},1;x)$ and $S_T^{-}(x) =
s(Z_1,\hdots,Z_{T-1},0;x)$. 
Then $S_T^{+}(x), S_T^{-}(x)\in \clZ_{T-1}$ and
\begin{align*}
  V_{T-1}(x) &=  \half\left(S_T^{+} (x)- S_{T}^-(x) \right),\quad x\in\bS,\\
  Y_{T-1}(x) &= \half\left(S_T^{+}(x) + S_{T}^-(x)\right) + c(x)
               (U_{T-1} + V_{T-1}(x)),\quad x\in\bS.
\end{align*}
\end{example}

\section{Proof of \Prop{prop:B_W_Gamma_R}}
\label{appdx:formula_for_R}

Based on the graphical model for the HMM (see \Fig{fig:hmm}), $X_{t+1} \perp\!\!\!\!\perp Z_{t+1} \mid \clF_t$.
Therefore,
\begin{equation*}
\E (B_{t+1}(f) \mid \clF_{t}
\vee \clZ_{t+1})  = \E (f(X_{t+1}) \mid \clF_{t}
\vee \clZ_{t+1}) - (Af)(X_t) = \E (f(X_{t+1}) \mid \clF_{t}) - (Af)(X_t) = 0,
\end{equation*}
and
\begin{align*}
\E(|B_{t+1}(f)|^2 \mid \clF_{t}
  \vee \clZ_{t+1}) & = \E(|f(X_{t+1}) - (Af)(X_t)|^2 \mid \clF_{t}) \\&= \E(|f(X_{t+1}) |^2 \mid \clF_{t})  - |(Af)(X_t)|^2  =
(\Gamma f)(X_t),
\end{align*}
for $t=0,1,2,\hdots, T-1$.  This completes the proof for the two formulae related to
the process $B$.  

Express the vector-valued functions $e:\bO\to\Re^m$ as
\begin{equation*}
e(z) = \begin{bmatrix} g_1(z) \\ g_2(z) \\ \vdots \\
  g_m(z) \end{bmatrix}_{m\times 1},\quad \text{where}\quad g_i(z)
= \begin{cases} -1, & z=0 \\ 1, & z=i \\ 0, & \text{o.w.}\end{cases},
\quad i=1,2,\hdots,m,
\end{equation*}
where note
\begin{equation*}
(C g_i)(x) = c(x),\quad x\in\bS,\quad i=1,2,\hdots, m.
\end{equation*}
Therefore,
\begin{equation*}
W_{t+1}(i) = g_i(Z_{t+1}) - (Cg_i)(X_t),\quad t=0,1,2,\hdots,T-1,\quad i=1,2,\hdots, m,
\end{equation*}
and because $\sP(Z_{t+1}=z \mid X_t=x) = C(x,z)$ for $x\in\bS$ and $z\in\bO$,
\begin{equation*}
\E (W_{t+1}(i) \mid \clF_t) = \sum_{z\in \bO} C(X_t,z) g_i(z) -
(Cg_i)(X_t) = 0, \quad t=0,1,2,\hdots,T-1,\quad i=1,2,\hdots, m.
\end{equation*}
This shows that $W$ is a martingale increment process.  It remains to
derive the formula for the $m\times m$ matrix $R$.  In order to
determine the $(i,j)$-entry of the matrix $R$, we consider
$\E(W_{t+1}(i) W_{t+1}(j) \mid \clF_t)$.  From the definition of $W$, 
\begin{align*}
  g_i(Z_{t+1}) &= (Cg_i)(X_t) + W_{t+1}(i) ,\quad
  t=0,1,2,\hdots,T-1,\quad i=1,2,\hdots, m, \\
  g_j(Z_{t+1}) &= (Cg_j)(X_t) + W_{t+1}(j) ,\quad t=0,1,2,\hdots,T-1,\quad j=1,2,\hdots, m.
\end{align*}
Then because $W$ is a martingale increment process,
\begin{align*}
\E (g_i(Z_{t+1})  g_j(Z_{t+1}) \mid \clF_t) & = (Cg_i)(X_t) (Cg_j)(X_t)
                                              + \E(W_{t+1}(i) W_{t+1}(j) \mid \clF_t)\\
  & = c(X_t) c(X_t)
  + \E(W_{t+1}(i) W_{t+1}(j) \mid \clF_t).
\end{align*}
Now, the left-hand side
\begin{align*}
  \E (g_i(Z_{t+1})  g_j(Z_{t+1}) \mid \clF_t) & = \sum_{z \in \bO}
                                                g_i(z) g_j(z)
                                                C(X_t,z) \\
  & = g_i(0)g_j(0) C(X_t,0)+\sum_{z =1}^m
                                                g_i(z) g_j(z)
                                                C(X_t,z) \\
  & = C(X_t,0) + \delta_{ij} C(X_t,i),
\end{align*}
where $\{\delta_{ij}:1\leq i\leq m,1\leq j\leq m\}$ is the Kronecker
$\delta$.  Combining,
\begin{equation*}
\E(W_{t+1}(i) W_{t+1}(j) \mid \clF_t)= C(X_t,0) + \delta_{ij} C(X_t,i)
-c(X_t) c(X_t),\quad 0\leq t\leq T-1,\quad 1\leq i,j\leq m.
\end{equation*}
Expressed in the matrix form,
\begin{equation*}
\E (W_{t+1} W_{t+1}^\tp \mid \clF_t) = R(X_t),\quad 0\leq t\leq T-1.
\end{equation*}
The meaning of $R(x)$ is as
 follows: 
 \begin{equation*}
 u^\tp R(x) u = C(x,0) (-1^\tp u)^2 + \sum_{i=1}^m C(x,i) (u(i))^2 -
 \left( C(x,0) (-1^\tp u) + \sum_{i=1}^m C(x,i) u(i) \right)^2,\;\;
 x\in \bS, \;\;u\in\Re^m,
\end{equation*}
where $1^\tp u = \sum_{i=1}^m u(i)$.  Note that for each fixed
$x\in\bS$, $C(x,\cdot)$ is a $1\times m$ probability vector.  
Therefore, $u^\tp R(x) u$ is the variance of $\begin{bmatrix} 
  (-\ones^\tp u) \\ u\end{bmatrix}\in \Re^{m+1}$, with respect to
the probability vector $C(x,\cdot)$. 

\section{Proof of \Thm{thm:duality-principle} (Duality principle)}
\label{appdx:duality-principle}

We prove a result for a more general class of estimators of the form
\begin{equation*}
		S_T = c_0- \sum_{t=0}^{T-1} U_t^\tp e(Z_{t+1}),
\end{equation*}
 where $U\in\clU$ and $c_0$ is a deterministic constant (note that in
 the statement of \Thm{thm:duality-principle}, $c_0=\mu(Y_0)$).  
For any such estimator, we show 
\begin{align}\label{eq:duality_formula_more_general}
		\E\big(|F(X_T)-S_T|^2\big) = \sJ_T(U;F) + (\mu(Y_0) - c_0)^2.
\end{align}
This more general formula is useful for another proof (see
Appendix~\ref{appdx:nonlinear_predictor_optimal_control}). 

For the HMM, we have defined two martingale increment processes (see \Sec{sec:prelim_mg}):
\begin{align*}
  B_{t+1}(f) &= f(X_{t+1}) - (Af)(X_t),\quad 0\leq t\leq T-1,\\
  W_{t+1} &= e(Z_{t+1})- c(X_t) ,\quad 0\leq t\leq T-1.
\end{align*}
Taking $f=Y_{t+1}$,
\begin{equation*}
 B_{t+1}(Y_{t+1}) = Y_{t+1}(X_{t+1}) - (A Y_{t+1})(X_t) ,\quad 0\leq t\leq T-1.
\end{equation*}
Since $Y$ solves BS$\Delta$E~\eqref{eq:dual_BSDE},
\begin{equation*}
(AY_{t+1})(x) = Y_{t}(x) - c^\tp(x) (U_t + V_t(x)) + V_t^\tp(x) e(Z_{t+1}),\quad
\forall x\in \bS ,\quad 0\leq t\leq T-1,
\end{equation*}
and therefore at $x=X_t$,
\begin{equation*}
(AY_{t+1})(X_t) = Y_{t}(X_t) -  c^\tp(X_t) (U_t + V_t(X_t)) + V_t^\tp(X_t) e(Z_{t+1}) ,\quad 0\leq t\leq T-1.
\end{equation*}
Because $W_{t+1} = e(Z_{t+1}) - c(X_t)$, we have
\begin{align*}
  c^\tp(X_t) U_t &= U_t^\tp e(Z_{t+1}) - U_t^\tp W_{t+1},\quad 0\leq t\leq T-1, \\
  c^\tp(X_t) V_t(X_t) &= V_t^\tp(X_t) e(Z_{t+1}) - V_t^\tp(X_t) W_{t+1} ,\quad 0\leq t\leq T-1,
\end{align*}
and thus,
\begin{equation*}
(AY_{t+1})(X_t) = Y_{t}(X_t) - U_t^\tp e(Z_{t+1}) + (U_t + V_t(X_t))^\tp W_{t+1},\quad 0\leq t\leq T-1.
\end{equation*}
Therefore,
\begin{align*}
  Y_{t+1}(X_{t+1}) & = (A Y_{t+1})(X_t) + B_{t+1}(Y_{t+1}) ,\quad 0\leq t\leq T-1 \\
  &= Y_{t}(X_t) - U_t^\tp e(Z_{t+1}) + (U_t + V_t(X_t))^\tp W_{t+1} + B_{t+1}(Y_{t+1}) ,\quad 0\leq t\leq T-1.
\end{align*}
Summing over $t=0,1,2,\hdots,T-1$, using the form
of the estimator $S_T$ and because $Y_T=F$,
\begin{equation*}
F(X_T) - S_T = (Y_0(X_0) -c_0) +\sum_{t=0}^{T-1} N_{t+1},
\end{equation*}
where
\begin{equation*}
N_{t+1} := (U_t + V_t(X_t))^\tp W_{t+1} + B_{t+1}(Y_{t+1}),\quad t=0,1,\hdots,T-1,
\end{equation*}
is a sum of two martingale increment processes.

Formula~\eqref{eq:duality_formula_more_general} is
obtained from squaring both sides and taking expectations.  Formulae
for the non-zero terms are as follows (for this purpose, denote $\clG_{t}:=\clF_{t-1}\vee \clZ_{t}$ for $t=1,2,\hdots, T$):
\begin{align*}
  \E  \left(|Y_0(X_0) - \mu(Y_0)|^2 \right) & = \var(Y_0(X_0)), \\
  \E \left( ((U_t^\tp + V_t^\tp(X_t)) W_{t+1})^2 \right)& = \E \left( (U_t^\tp + V_t^\tp(X_t))
                                     \E(W_{t+1} W_{t+1}^\tp|\clF_t)
                                                          (U_t + V_t(X_t))\right), \\
                                 &=\E \left( (U_t^\tp + V_t^\tp(X_t))
                                     R(X_t)
                                                          (U_t + V_t(X_t))\right) \\
\E \left( (B_{t+1}(Y_{t+1}))^2 \right)  &= \E \left(
                                          \E((B_{t+1}(Y_{t+1}))^2
                                          |\clG_{t+1})\right) =
                                          \E((\Gamma Y_{t+1})(X_t)).
\end{align*}
The cross-terms are zero because
\begin{align*}
                                                                     \E
                                                                     \left((Y_0(X_0)
                                                                     -
                                                                     \mu(Y_0))
                                                                     (U_t^\tp
                                                                     +
                                                                     V_t^\tp(X_t))
                                                                     \E(W_{t+1}|\clF_t)
                                                                     \right)
                                                                     &=0, \\
     \E  \left((Y_0(X_0) - \mu(Y_0)) \E(B_{t+1}(Y_{t+1}) |\clG_{t+1}
     )\right) &= 0,\\
    \E \left((U_t^\tp+V_t^\tp (X_t)) W_{t+1} \E(B_{t+1}(Y_{t+1})
  |\clG_{t+1})\right) &=0,
\end{align*}
and for $\tau > t$,
\begin{align*}
 \E\left((U_t^\tp + V_t^\tp(X_t)) W_{t+1} (U_\tau + V_\tau(X_\tau)) \E(W_{\tau+1}
  |\clF_\tau)\right) &= 0 ,\\
  \E \left(B_{t+1}(Y_{t+1}) \E((U_\tau^\tp+V_\tau^\tp(X_\tau))
  W_{\tau+1}|\clF_{\tau})\right) &=0 ,
  \\
  \E \left( B_{t+1}(Y_{t+1})
  \E(B_{\tau+1}(Y_{\tau+1})|\clG_{\tau+1})\right) &= 0, \\
      \E \left((U_t^\tp+V_t^\tp(X_t)) W_{t+1} \E(B_{\tau+1}(Y_{\tau+1})
  |\clG_{\tau+1})\right) &=0.
\end{align*}

\section{Proof of \Prop{prop:nonlinear_predictor_optimal_control}}
\label{appdx:nonlinear_predictor_optimal_control}

From the existence result described
in~\Prop{prop:existence_nonlin_predictor_rep}, there exists a 
$c^*\in\Re$ and $U^*=\{U_t^*:0\leq t\leq T-1\}\in\clU$ such that
  \begin{equation*}
   \pi_T(F) = c^*- \sum_{t=0}^{T-1} (U_t^*)^\tp e(Z_{t+1}),\quad \sP\text{-a.s.}
 \end{equation*}  
 From the duality principle (see the formula~\eqref{eq:duality_formula_more_general} shown in
  Appendix~\ref{appdx:duality-principle}),
  \begin{equation*}
\E (|F(X_T) - \pi_T(F)|^2) = \sJ_T(U^*;F) + (c^* - \mu(Y_0^*))^2,
\end{equation*}
where $Y_0^*$ is obtained from solving the BS$\Delta$E~\eqref{eq:dual_BSDE} with control
$U=U^*$.

We show that $U^*$ is an optimal control.  Suppose there exists a
$\tilde{U}\in \clU$  such that 
\begin{equation*}
\sJ_T(U^*;F)  \geq \sJ_T(\tilde{U};F).
\end{equation*}
Then set $\tilde{S} = \mu(\tilde{Y}_0) - \sum_{t=0}^{T-1}
  \tilde{U}_t^\tp e(Z_{t+1})$ where $\tilde{Y}_0$ is obtained from solving the BS$\Delta$E~\eqref{eq:dual_BSDE} with control
$U=\tilde{U}$.  Then
  using the duality principle,
  \begin{equation*}
\sJ_T(\tilde{U};F) = \E (|F(X_T) - \tilde{S}|^2) \geq \E (|F(X_T) - \pi_T(F)|^2),
\end{equation*}
where the inequality is from the MMSE property of the conditional expectation.

Combining,
\begin{align*}
  \E (|F(X_T) - \pi_T(F)|^2) &= \sJ_T(U^*;F) + (c^* - \mu(Y_0^*))^2 \\
                             & \geq  \sJ_T(\tilde{U};F) + (c^* - \mu(Y_0^*))^2 \\
& = \E (|F(X_T) - \tilde{S}|^2) + (c^* - \mu(Y_0^*))^2
  \\
  & \geq \E (|F(X_T) - \pi_T(F)|^2) + (c^* - \mu(Y_0^*))^2 .
\end{align*}   
It then follows that $c^* = \mu(Y_0^*)$ and all the inequalities
  are in fact equalities.  In particular, 
  \begin{equation*}
\sJ_T(\tilde{U};F) = \sJ_T(U^*;F)  = \E (|F(X_T) - \pi_T(F)|^2) = \text{MMSE}.
\end{equation*}
This proves existence---both $\tilde{U}$ and $U^*$ are optimal controls
that attain the optimal value given by MMSE.  Moreover,
\begin{equation*}
\E (|F(X_T) - \tilde{S}|^2) = \E (|F(X_T) - \pi_T(F)|^2) \implies
\tilde{S} = \pi_T(F),\quad \sP\text{-a.s.}.
\end{equation*}
because of the uniqueness property of the conditional expectation.

\begin{remark}\label{rem:uniqueness_of_U}
To conclude uniqueness (that is, $U^*=\tilde{U}$, $\sP$-a.s.) requires
additional assumption on the model.  For example, a sufficient
condition for the same is to assume $C(x,z)>0$ for all $x\in\bS$ and $z\in\bO$.
Then the proof of \Prop{prop:existence_nonlin_predictor_rep} shows
that $U^*$ is unique (see
Appendix~\ref{proof:prop:existence_nonlin_predictor_rep}).  Because
$U^*$ is an optimal control input, it follows that the optimal
control input is unique. 
\end{remark}

\section{Proof of  \Thm{thm:optimal-solution}}
\label{appdx:proof_of_optimal-solution}

The formula for the optimal control is easiest to see from the consideration of the OCP~\eqref{eq:dual-optimal-control} for $T=1$. 

\subsection{Formula for the optimal control for $T=1$}
With $T=1$, the OCP~\eqref{eq:dual-optimal-control} is
\begin{align*}
\min_{U_0\in \Re^m} \quad \sJ_1(U_0;F)  &=
\mu(Y_0^2) - \mu(Y_0)^2 + 
                                        \E \Big( l (F,V_0,U_0\,;X_0)
                                        \Big), \\
\text{subject to} \quad Y_0(x) &= (AF)(x) + c^\tp(x) (U_0 + V_0(x)) -
                                 V_0^\tp(x) e(Z_1),\quad x\in\bS.
\end{align*}
Because $F(x)\in \clZ_1$, there exists the deterministic $s:\bS\times
\bO\to\Re$ such that
\begin{equation*}
F(x) = s(x,Z_1),\quad x\in\bS.
\end{equation*}
Define
\begin{equation*}
f(x):=\frac{1}{m+1} \sum_{z\in\bO} s(x,z),\quad \tilde{s}_i(x):=
s(x,i) - f(x),\quad i=1,2,\hdots,m,\quad x\in\bS.
\end{equation*}
Then using \Prop{prop:linear_algebra}, 
\begin{equation*}
s(x,z) = f(x) + \tilde{s}(x) e(z),\quad z\in \bO,\quad x\in \bS,
\end{equation*}
where $\tilde{s}(\cdot)=\begin{bmatrix} \tilde{s}_1(\cdot) &\tilde{s}_2(\cdot)
  &\hdots &\tilde{s}_m(\cdot)  \end{bmatrix}_{d\times m}$. Therefore,
\begin{equation*}
F(x) = s(x,Z_1) = f(x) + \tilde{s}(x) e(Z_1),\quad x\in \bS.
\end{equation*}
Based on the decomposition above,
\begin{align*}
  (AF)(x) &= (Af)(x) + (A\tilde{s})(x) e(Z_1),\quad x\in \bS, \\
  &= Y_0(x) - c^\tp(x) (U_0 + V_0(x)) + V_0^\tp(x) e(Z_1),\quad x\in\bS,
\end{align*}
which gives
\begin{align*}
  V_0(x) &= (A\tilde{s})^\tp(x) ,\quad x\in\bS, \\
  Y_0(x) &= (Af)(x) + c^\tp(x) (U_0 + V_0(x)),\quad x\in\bS.
\end{align*}
Based on this, the OCP reduces to a standard (deterministic) linear
quadratic (LQ) problem
\begin{align*}
\min_{U_0\in \Re^m} \quad \sJ_1(U_0;F)  &=
\mu(Y_0^2) - \mu(Y_0)^2 + 
                                         \sum_{x} \mu(x) (U_0+V_0(x))^\tp
                                          R(x) (U_0 + V_0(x))
                                          + \E (\Gamma F(X_0)), \\
\text{subject to} \quad Y_0(x) &= (Af)(x) + c^\tp(x) (U_0 + V_0(x)),\quad x\in\bS.
\end{align*}
where note $\E (\Gamma F(X_0))$ is not affected by the control
$U_0$. The LQ problem is readily solved to obtain the formula for the optimal
control,
\begin{equation*}
U_0^\opt = \phi(Y_0^\opt,V_0;\mu),
\end{equation*}
where
\begin{equation*}
Y_0^\opt(x) = (Af)(x) + c^\tp(x) (U_0^\opt + V_0(x)),\quad x\in\bS.
\end{equation*}
Let $U_0 = U_0^\opt + \tilde{U}_0$ then
\begin{equation*}
Y_0(x) = Y_0^\opt(x) + c^\tp(x) \tilde{U}_0 ,\quad x\in\bS,
\end{equation*}
A standard completion-of-square argument is used to show that
\begin{align*}
 \sJ_1(U_0;F) &=  \sJ_1(U_0^\opt;F) +  \mu((c^\tp \tilde{U}_0)^2) - \mu(c^\tp \tilde{U}_0)^2 + 
                                         \sum_{x} \mu(x) \tilde{U}_0^\tp
                R(x) \tilde{U}_0 \nonumber \\
  &=\sJ_1(U_0^\opt;F) + \langle \tilde{U}_0,\tilde{U_0}\rangle_{p_0},\label{eq:last_step_formula}
\end{align*}
where $p_0=\mu (C)$.  Because the calculations are entirely identical
also for the general case, these are included in
\Sec{sec:calc_for_extra} at the end of this proof. 


From 
\Prop{prop:nonlinear_predictor_optimal_control}, 
\begin{equation*}
\sJ_1(U_0^\opt;F) = \text{MMSE} = \E( |F(X_1)- \pi_1(F)|^2) = \E (\pi_1(F^2)- \pi_1(F)^2),
\end{equation*}
where the last equality is from the use of the tower property.
Summarizing, 
\begin{equation}\label{eq:base_case}
\sJ_1(U_0;F) = \E (\pi_1(F^2)- \pi_1(F)^2) + \langle U_0-U_0^\opt, U_0-U_0^\opt \rangle_{p_0}.
\end{equation}

\subsection{Proof of \Thm{thm:optimal-solution}}

\newP{Notation} 
For a function $f\in \Re^d$, and $t\in\mathbb{T}$, denote $
\clV_t(f) := \pi_t(f^2) - \pi_t(f)^2
$. Note $\clV_t(f)$
represents the conditional variance of the random variable $f(X_t)$ 
because $\clV_t(f)=\E(|f(X_t)-\pi_t(f)|^2 \mid \clZ_t)$.  

\medskip

\begin{proof}[of  \Thm{thm:optimal-solution}]
Define
\begin{align*}
\sJ_1(U_0;Y_1)  &:=
\mu(Y_0^2) - \mu(Y_0)^2 + 
                                        \E \Big( l (Y_1,V_0,U_0\,;X_0)
                \Big), \\
\sJ_{t+1}(U_0,\hdots,U_{t};Y_{t+1}) &:=
\sJ_{t}(U_0,\hdots,U_{t-1};Y_{t}) + \E (l(Y_{t+1},V_t,U_t;X_t) ),\quad t=1,2,\hdots,T-1.
\end{align*}
Note that at the terminal time, this gives the optimal control objective.

The proof of~\eqref{eq:optimal_control_value} is by induction based on
essentially repeating the calculations described for
$T=1$ in the preceding subsection.  Suppose it has already been shown that
\begin{equation*}
\sJ_{t}(U_0,\hdots,U_{t-1};Y_{t}) = \E \left( \sum_{s=0}^{t-1}
\langle \tilde{U}_s,\tilde{U}_s \rangle_{p_s} \right)+ \E ( \clV_t(Y_t)).
\end{equation*}
Our task is to show that 
\begin{equation*}
\sJ_{t+1}(U_0,\hdots,U_{t};Y_{t+1}) = \E \left( \sum_{s=0}^{t}
 \langle \tilde{U}_s,\tilde{U}_s \rangle_{p_s} \right) + \E ( \clV_{t+1}(Y_{t+1})),
\end{equation*}
where $\tilde{U}_s={U}_s - U_s^\opt$ for $0\leq s\leq t$.  At the
terminal time, this is the desired formula~\eqref{eq:optimal_control_value}.  

The base case ($t=1$) is proved in~\eqref{eq:base_case}.  Using the
induction hypothesis, we have
\begin{align*}
\sJ_{t+1}(U_0,\hdots,U_{t};Y_{t+1}) &=
                                      \sJ_{t}(U_0,\hdots,U_{t-1};Y_{t}) + \E (l(Y_{t+1},V_t,U_t;X_t) )\\
  &= \E \left( \sum_{s=0}^{t-1}
 \langle \tilde{U}_s,\tilde{U}_s \rangle_{p_s} \right) + \E ( \clV_t(Y_t)) + \E (l(Y_{t+1},V_t,U_t;X_t) )
\end{align*}
Now,
\begin{equation*}
  \E ( \left(\clV_t(Y_t) + l(Y_{t+1},V_t,U_t;X_t) \right) \mid \clZ_t ) = \clV_t(Y_t)
  + \sum_{x\in\bS} \pi_t(x)
(U_t + V_t(x))^\tp R(x) (U_t + V_t(x)) +
  \pi_t(\Gamma Y_{t+1}) ,
\end{equation*}
where the $(Y_t,V_t,U_t)\in\clZ_t$ are related to $Y_{t+1}\in\clZ_{t+1}$ via
the BS$\Delta$E,
\begin{equation*}
Y_t(x) = (AY_{t+1})(x) + c^\tp(x) (U_t + V_t(x)) - V_t^\tp(x)e(Z_{t+1}),\quad x\in\bS.
\end{equation*}
From the theory for BS$\Delta$E described in
Appendix~\ref{proof:prop:existence_BSDE}, for any given $Y_{t+1}\in
\clZ_{t+1}$, there exists a unique such $V_t\in \clZ_t$ such that
above holds.  Set
\begin{equation*}
U_t^\opt = \phi(Y_t^\opt,V_t;\pi_t),
\end{equation*}
where
\begin{equation*}
Y_t^\opt(x) = (AY_{t+1})(x) + c^\tp(x) (U_t^\opt + V_t(x)) - V_t^\tp(x)e(Z_{t+1}),\quad x\in\bS.
\end{equation*}
Let $U_t = U_t^\opt + \tilde{U}_t$ then
\begin{equation*}
Y_t(x) = Y_t^\opt(x) + c^\tp(x) \tilde{U}_t ,\quad x\in\bS,
\end{equation*}
A completion-of-square argument then gives (the calculations for the
same are included in \Sec{sec:calc_for_extra} at the end of this proof),
\begin{align*}
\clV_t(Y_t) & + \sum_{x\in\bS} \pi_t(x)
(U_t + V_t(x))^\tp R(x) (U_t + V_t(x)) 
\\
& =  \langle \tilde{U}_t,\tilde{U}_t \rangle_{p_t}  +\clV_t(Y_t^\opt) + \sum_{x\in\bS} \pi_t(x)
(U_t^\opt + V_t(x))^\tp R(x) (U_t^\opt + V_t(x)),
\end{align*}
and thus, upon adding $\pi_t(\Gamma Y_{t+1})$ to both sides,
\begin{align*}
  \clV_t(Y_t) + \E \big( l (Y_{t+1},V_t,U_t\,;X_t)|\clZ_t\big) =
  \langle \tilde{U}_t,\tilde{U}_t \rangle_{p_t} + \clV_t(Y_t^\opt) + \E \big( l (Y_{t+1},V_t,U_t^\opt\,;X_t)|\clZ_t\big).
\end{align*}
Therefore,
\begin{align*}
\sJ_{t+1}(U_0,\hdots,U_{t};Y_{t+1}) &=
                                      \sJ_{t}(U_0,\hdots,U_{t-1};Y_{t}) + \E (l(Y_{t+1},V_t,U_t;X_t) )\\
  &= \E \left( \sum_{s=0}^{t-1}
 \langle \tilde{U}_s,\tilde{U}_s \rangle_{p_s}  \right) + \E ( \clV_t(Y_t)) + \E
    (l(Y_{t+1},V_t,U_t;X_t) )\\
  &= \E \left( \sum_{s=0}^{t}
 \langle \tilde{U}_s,\tilde{U}_s \rangle_{p_s} \right) + \E(\clV_t(Y_t^\opt)) + \E \big( l
    (Y_{t+1},V_t,U_t^\opt\,;X_t) \big).
\end{align*}
From~\Prop{prop:nonlinear_predictor_optimal_control}, 
\begin{equation*}
\sJ_{t+1}(U_0^\opt,\hdots,U_t^\opt;Y_{t+1})  = \E(\clV_{t+1}(Y_{t+1})),
\end{equation*}
and we have established the induction formula for $t+1$.  Also
from~\Prop{prop:nonlinear_predictor_optimal_control},
\begin{equation*}
\pi_{t+1}(Y_{t+1}) = \mu (Y_0^\opt) -
                   \sum_{s=0}^{t} (U_s^\opt)^\tp e(Z_{s+1}),
                 \end{equation*}
                 which proves~\eqref{eq:estimator-t}.
\end{proof}

\subsection{Details of the completion-of-square calculation}
\label{sec:calc_for_extra}

Let $U_t = U_t^\opt + \tilde{U}_t$ and $Y_t = Y_t^\opt + \tilde{Y}_t$
where $\tilde{Y}_t(x) = c^\tp (x) \tilde{U}_t$.  Then
\begin{align*}
 \clV_t(Y_t) & +  \sum_{x\in\bS} \pi_t(x)
(U_t + V_t(x))^\tp R(x) (U_t + V_t(x))  = \left(
   \clV_t(\tilde{Y}_t) + \sum_{x\in\bS} \pi_t(x)
(\tilde{U}_t)^\tp R(x) \tilde{U}_t\right) 
\\
& \qquad +  \left( \clV_t(Y_t^\opt) + \sum_{x\in\bS} \pi_t(x)
  (U_t^\opt + V_t(x))^\tp R(x) (U_t^\opt + V_t(x)) \right) + \text{(cross-term)} ,
\end{align*}
where the cross-term is given by,
\begin{align*}
\text{(cross-term)}  &= 2\left( \pi_t( Y_t^\opt c^\tp) \tilde{U_t} - \pi_t( Y_t^\opt) \pi_t(
  c^\tp) \tilde{U_t}  + \sum_{x\in\bS} \pi_t(x) 
  (U_t^\opt + V_t(x))^\tp R(x) \tilde{U}_t \right)\\
  & = 2 \, \tilde{U_t}^\tp \left( \pi_t((c-\pi_t(c))Y_t^\opt) + \pi_t(R)
    U_t^\opt + \pi_t(R V_t) \right) = 0.
\end{align*}
It remains to show that
\begin{equation*}
\clV_t(\tilde{Y}_t) + \sum_{x\in\bS} \pi_t(x) (\tilde{U}_t)^\tp R(x)
\tilde{U}_t = \langle \tilde{U}_t,\tilde{U}_t \rangle_{p_t}.
\end{equation*}
Because $\tilde{Y}_t(x) = c^\tp (x) \tilde{U}_t$,
\begin{equation*}
\clV_t(\tilde{Y}_t) + \sum_{x\in\bS} \pi_t(x) (\tilde{U}_t)^\tp R(x)
\tilde{U}_t = \tilde{U}_t^\tp \pi_t (c c^\tp) \tilde{U}_t-
\left(\pi_t(c)^\tp \tilde{U}_t\right)^2 +  \tilde{U}_t^\tp (\pi_t(R)) \tilde{U}_t.
\end{equation*}
From the definition of $R(x)$,
\begin{align*}
  R(x) + c(x) c^\tp(x) &= \text{diag}(c(x)) + C(x,0) (I + \ones
                         \ones^\tp),\quad x\in\bS\\
  \therefore,\quad \pi_t(R+c c^\tp) &=\text{diag}(\pi_t(c)) + p_t(0) (I + \ones
                                      \ones^\tp).
\end{align*}
Then
\begin{align*}
\tilde{U}_t^\tp (\pi_t(R+c c^\tp)) \tilde{U}_t &=
                                                                     \sum_{i=1}^m
                                                                     (
                                                                     (\pi_t(c))(i)
                                                                     +
                                                                     p_t(0))
                                                                     (\tilde{U}_t(i))^2+
                                                                     p_t(0)
                                                                     (\sum_{i=1}^m
                                                                     \tilde{U}_t(i))^2
  \\
 &= \sum_{i=1}^m  p_t(i) (\tilde{U}_t(i))^2+ p_t(0) (\sum_{i=1}^m \tilde{U}_t(i))^2.    
\end{align*}
Finally,
\begin{equation*}
\pi_t(c)^\tp \tilde{U}_t = \sum_{i=1}^m
(\pi_t(c))(i) \tilde{U}_t(i)) = \sum_{i=1}^m p_t(i) \tilde{U}_t(i)) +
p_t(0) (-\sum_{i=1}^m \tilde{U}_t(i)).
\end{equation*}
and the result follows.

\section{Proof of~\Prop{prop:optimal_control_formula_T=1} (Example with $m=1$ and $T=1$)}
\label{sec:simple_case}

\begin{figure}[h]
\begin{center}
    \begin{tikzpicture}[node distance=2cm, auto]
        \node (X0) [thick, circle, draw] {\(X_0\)};
        \node (X1) [thick ,circle, draw, right=of X0] {\(X_1\)};
        \node (Z1) [thick, rectangle, draw, above=of X1] {\(e(Z_1)\)}; 
        
        \draw[thick,->] (X0) -- (X1);
        \draw[thick,->] (X0) -- (Z1);
        
        \draw[thick,->] ([yshift=-1cm]X0.west) -- ([yshift=-1cm]X1.east); 
        \node[below=0.8cm of X0] {\(t=0\)};
        \node[below=0.8cm of X1] {\(t=1\)};
        \draw ([yshift=-0.7cm]X0.south) -- ++(0,0.2); 
        \draw ([yshift=-0.7cm]X1.south) -- ++(0,0.2); 
    \end{tikzpicture}
  \end{center}
  \vspace*{-20pt}
  \caption{Graphical model for the HMM for $T=1$.}\label{fig:HMM_T=1}
  \end{figure}

For $T=1$, there are only three random variables $(X_0,X_1,e(Z_1)) \in \bS\times \bS \times
\{-1,1\}$ whose relationship is depicted in \Fig{fig:HMM_T=1}. Our
interest is to compute the conditional expectation
$\pi_1(F)=\E(F(X_1)|\clZ_1)$ for $F\in\clZ_1$.  The formula for the
same is given by,
\begin{equation}\label{eq:one_step_filter_formula}
\pi_{1}(F) = \begin{cases} \frac{\mu (c^+ (A F))}{\mu 
    (c^+)}, & \text{if} \; e(Z_{1})=1, \\[5pt]
  \frac{\mu (c^- (A F))}{\mu
    (c^-)}, & \text{if} \; e(Z_{1})=-1. \end{cases} 
\end{equation}
The goal is to derive this formula from solving the OCP~\eqref{eq:dual-optimal-control} with $T=1$.

With $T=1$, the OCP~\eqref{eq:dual-optimal-control} is an example of a
single-stage OCP as follows:
\begin{align*}
&\min_{u_0\in\Re} \quad \eta(u_0):=\mu(y_0^2) - \mu(y_0)^2  + \mu((u_0+v_0)^2r) +
  \E((\Gamma F)(X_0)) \\
&\text{subject to:}\quad y_0(x) = (AF)(x) + c(x) (u_0 + v_0(x)) - v_0(x)e(Z_1),\quad x\in\bS,
\end{align*}
where we use the lower case notation for $u_0,y_0,v_0$ to stress the
fact that these are all deterministic (the cost
$\eta(u_0)=\sJ_1(U_0;F)$ for $U_0=u_0$).  The two random variables 
$F(\cdot)$ and $Z_1$ are denoted by capital letters. 
The third term in the objective, $\E((\Gamma F)(X_0))$, is
not affected by the choice of the decision variable $u_0$.

The single-stage OCP is to choose a real number $u_0$ to minimize
the quadratic objective, subject to a linear
constraint.  The only reason that the
solution is not entirely elementary is because the constraint is random.   

Let
\[
  F(x) = \begin{cases} F^+(x), & \text{if}\; e(Z_1)=1, \\
    F^-(x), & \text{if}\; e(Z_1)=-1. \end{cases}
\]
Then the constraint is given by two deterministic equations
\begin{align*}
  y_0(x) &= (AF^+)(x) + c(x) (u_0 + v_0(x)) - v_0(x) ,\quad x\in\bS,\\
  y_0(x) &= (AF^-)(x) + c(x) (u_0 + v_0(x))  + v_0(x) ,\quad x\in\bS,
\end{align*}
which are readily solved to obtain formulae for the solution $(y_0,v_0)$ as
follows:
\begin{align*}
  v_0(x) &= (A\tilde{f})(x), \quad x\in\bS,\\
  y_0(x) &= (Af)(x) + c(x) (u_0 + v_0(x)), \quad x\in\bS,
\end{align*}
where $f(x) = \frac{F^+(x)+F^-(x)}{2}$ and $\tilde f(x) =
\frac{F^+(x)-F^-(x)}{2}$. For the particular case where $F$ is
deterministic, $v_0=0$. 

 To solve the optimization problem, we consider two cases as follows:
\begin{itemize}
\item The case where $\mu(r)\neq 0$ (equivalently, $1-\mu(c)^2 \neq 0$).
\item The case where  $\mu(r) = 0$ (equivalently, $1-\mu(c)^2 = 0$). 
\end{itemize}

\newP{$\bullet$ Case where $\mu(r)\neq 0$} Substitute the solution
for $y_0$ into the quadratic objective, upon taking the
derivative and setting it to zero, the formula for optimal control is
given by,
\begin{align*}
  u_0 = u_0^\opt & :=
        \frac{-1}{\mu(r)} \left( \mu(y_0(c-\mu(c))) + \mu
  (v_0 r) \right) =  \frac{-1}{1-\mu(c)^2} (R_1+R_2),
\end{align*}
where $R_1=\left(\mu((Af)c) - \mu(Af)\mu(c) \right)$ and
$R_2=\left(\mu(v_0) - \mu(v_0 c) \mu(c) \right)$.  This is the
formula~\eqref{eq:optimal_control_formula_T=1}.  
Moreover, a completion-of-square argument gives 
\[
\eta(u_0) =  \eta (u_0^\opt) + (1-\mu(c)^2) (u_0-u_0^\opt)^2,
\]
which shows that the optimal control is unique.

The resulting estimator is given by
\begin{align*}
  S_1 &= \mu(Af) + \mu(c(u_0+v_0)) + \frac{R_1+R_2}{1-\mu(c)^2} \, e(Z_1) \\
  &= \begin{cases}
  \mu(Af) + \mu(c v_0) + \frac{1-\mu(c)}{1-\mu(c)^2} (R_1+R_2), &
  \text{if}\;\; e(Z_1)=1, \\[5pt]
     \mu(Af) + \mu(c v_0) - \frac{1+\mu(c)}{1-\mu(c)^2} (R_1+R_2), &
  \text{if}\;\; e(Z_1)=-1.
    \end{cases}
\end{align*}
Using the identities $\mu(c^+- c^-)=\mu(c)$ and $\mu(c^++c^-)=1$, upon
simplifying, 
\begin{align*}
&\mu(Af) + \frac{R_1}{2\mu(c^+)} = \frac{\mu(c^+(Af))}{\mu(c^+)}
  , \quad \mu(c v_0) + \frac{R_2}{2\mu(c^+)}  =  \frac{\mu(c^+
  v_0)}{\mu(c^+)}, \\
& \mu(Af) - \frac{R_1}{2\mu(c^-)} = \frac{\mu(c^-(Af))}{\mu(c^-)}
  , \quad \mu(c v_0) - \frac{R_2}{2\mu(c^-)}  =  \frac{-\mu(c^-
  v_0)}{\mu(c^-)}.
\end{align*}
Upon combining the terms,
\[
S_1 = \begin{cases}
   \frac{\mu(c^+(Af+v_0))}{\mu(c^+)}, &
  \text{if}\;\; e(Z_1)=1, \\[5pt]
    \frac{\mu(c^-(Af-v_0))}{\mu(c^-)}, &
  \text{if}\;\; e(Z_1)=-1 .
    \end{cases}
\]
This coincides with the formula~\eqref{eq:one_step_filter_formula}
because $Af + v_0 = AF^+$ and $Af - v_0 = AF^-$.  From the duality
principle,
\[
\eta (u_0^\opt) =       \E (|F(X_1) - \pi_1(F)|^2).
\]

\newP{$\bullet$ Case where $\mu(r) = 0$} Set the control and
the estimator as 
\[
u_0 = 0,\quad S_1 = \mu(y_0).
\]
Because $1-\mu(c)^2 = 4 \mu(c^+) \mu(c^-)$, $1-\mu(c)^2 = 0$ iff
either $\mu(c^+) =0$ or $\mu(c^-)=0$.  We have
\begin{align*}
\mu(c^+) &= 0 \implies c(x) = c^+(x) - c^{-}(x) = -1,
  \;\;\forall\;x\in\text{supp}(\mu),\\
\mu(c^-) &= 0 \implies c(x) = c^+(x) - c^{-}(x) = +1,
  \;\;\forall\;x\in\text{supp}(\mu).
\end{align*}
Using the formula for $y_0$, this gives
\begin{align*}
\mu(c^+) &= 0 \implies S_1 = \mu(Af - v_0) = \mu(AF^{-}),\\
\mu(c^-) &= 0 \implies S_1 = \mu(Af + v_0) = \mu(AF^{+}).
\end{align*}
Because $\mu(c^\pm) = \sP(Z_1=\pm1)$,
\[
S_1 = \pi_1(F),\quad \sP\text{-a.s.}
\]

\section{Proof of \Prop{prop:dual_filter}}
\label{appdx:prop_dual_filter}

For $T=1$, the result follows from noting that with $Y_1=f$ (where $f$
is deterministic) the optimal BS$\Delta$E reduces to a BDE
\begin{equation*}
y_0 = (Af)(x) + c^\tp(x)  \phi(y_0,0;\mu),\quad (\because,\;\;V_0=0).
\end{equation*}
From \Thm{thm:optimal-solution},
\begin{equation*}
\pi_1(f) = \mu(y_0) - \phi(y_0,0;\mu)^\tp e(Z_1),\quad
\sP\text{-a.s.}, \quad f\in\Re^d
\end{equation*}
and therefore,
\begin{equation*}
\pi_1^{(z_1)}(f) = \mu(y_0) - \phi(y_0,0;\mu)^\tp e(z_1),\quad f\in\Re^d
\end{equation*}
Reduction to $m=1$ is
most readily seen from noting the formula for the conditional measure,
\begin{equation*}
\pi_1^{(z_1)}(f) = \frac{\pi(C(:,z_1) (Af))}{\pi(C(:,z_1))},\quad f\in \Re^d.
\end{equation*}
Because the right-hand side only depends upon $C(:,z_1)$, for the
purpose of computing $\pi_1^{(z_1)}(f)$, we can 
consider the fixed-point equation for a binary-valued HMM.


For $T=2$, repeating
the argument above,
\begin{equation*}
\pi_2^{(z_1,z_2)}(f) = \pi_1^{(z_1)}(y_1) -
\phi(y_1,0;\pi_1^{(z_1)})^\tp e(z_2),\quad (\because,\;\;
\pi_1^{(z_1)} \;\text{is deterministic}),
\end{equation*}
and therefore,
\begin{equation*}
\pi_2^{(z_1,z_2)}(f) = \mu(y_0) - \phi(y_0,0;\mu)^\tp e(z_1) -
\phi(y_1,0;\pi_1^{(z_1)})^\tp e(z_2).
\end{equation*}
The general case follows by induction.

\section{Pseudocodes for dual filter algorithms}
\label{appdx:algorithms}

\begin{algorithm}
  \caption{Dual filter $\clN^\dfltr$ (iterative)}
  \label{alg:dual_filter_iterative}
  \begin{algorithmic}[1]
  \REQUIRE HMM parameters $(A,C)$, observation sequence $z
  =[z_1,z_2,\hdots,z_T]\in\bO^T$, measure
  $\rho=\{\rho_0,\rho_1,\hdots,\rho_{T-1}\}\in \Re^{d\times T}$
  \ENSURE Measure $\rho^+ = \{\rho_1^+,\hdots,\rho_{T-1}^+,\rho_{T}^+\}\in \Re^{d\times T}$
  
  \STATE $T \gets$ length of $z$
  \STATE $f \gets I_d$ \hfill \% Eq.~\eqref{eq:opt_BDE_b} (set f to identity matrix)
  
  \FOR{$t = T$ to $1$} 
      \STATE $c \gets 2 \cdot C(:,z_{t}) - 1$ \hfill \% Set $c= c^+-c^-$
      \STATE $u \gets \text{compute\_optimal\_control}(\rho_{t-1}, {A}
      f, 0, c)$ \hfill \% Eq.~\eqref{eq:opt_BDE_a2} (Algorithm~\ref{alg:compute_u})
      \STATE $f \gets A f + c \cdot u$ \hfill \% Eq.~\eqref{eq:opt_BDE_a}
      \STATE $u_{t-1} \gets u$
      \STATE $y_{t-1} \gets f$
  \ENDFOR
  
  \STATE $s \gets  \rho_0 \cdot y_0$
  
  \FOR{$t = 1$ to $T$}
  \STATE $s \gets  s - u_{t-1}$
  \STATE $\rho^+_t \gets s \cdot y_t^{\dagger}$ \hfill \% Eq.~\eqref{eq:opt_BDE_c}
  \STATE $\rho^+_t \gets \text{normalize\_distribution}(\rho_t^+)$
  \ENDFOR
  \RETURN $\rho^+$
  
  \end{algorithmic}
\end{algorithm}

\begin{algorithm}
  \caption{Dual filter $\clN^\dfltr$ (single-shot)}
  \label{alg:dual_filter_single_shot}
  \begin{algorithmic}[1]
  \REQUIRE HMM parameters $(A,C)$, initial measure $\mu$, observation sequence $z
  =[z_1,z_2,\hdots,z_T]\in\bO^T$
  \ENSURE Measure $\rho^+ = \{\rho_1^+,\rho_2^+,\hdots,\rho_{T}^+\}\in \Re^{d\times T}$
  
  \STATE $T \gets$ length of $z$
  
  \FOR{$t = 1$ to $T$}
    \STATE $f \gets I_d$ \hfill \% Eq.~\eqref{eq:opt_BDE_b} (set f to identity matrix)
    \STATE $s \gets 0$
    \FOR{$\tau = t$ to $1$}
      \STATE $c \gets 2 \cdot C(:,z_{\tau}) - 1$ \hfill \% Set $c= c^+-c^-$
      \STATE $u \gets \text{compute\_optimal\_control}(\rho^+_{\tau-1}, {A}f, 0, c)$ \hfill \% Eq.~\eqref{eq:opt_BDE_a2} (Algorithm~\ref{alg:compute_u})
      \STATE $f \gets A f + c \cdot u$ \hfill \% Eq.~\eqref{eq:opt_BDE_a}
      \STATE $s \gets s-u$
    \ENDFOR
    \STATE $s \gets  s+\mu \cdot f$
    \STATE $\rho^+_t \gets \text{normalize\_distribution}(s)$
  \ENDFOR
  
  \RETURN $\rho^+$
  
  \end{algorithmic}
\end{algorithm}

\begin{algorithm}
\caption{Computation of control input $u$}
\label{alg:compute_u}
\begin{algorithmic}[1]
\REQUIRE Measure $\rho$, functions $g$, $v$, and $c$ (all are elements
of $\Re^d$)
\ENSURE Control $u\in\Re$
\IF{\(\rho(c)^2= 1\)}
    \STATE $u \gets 0$
\ELSE
    \STATE $u \gets -\frac{1}{1 - \rho(c)^2} \cdot \big( (\rho(g \cdot c) - \rho(g) \rho(c)) + (\rho(v) - \rho(v \cdot c) \rho(c)) \big)$
\ENDIF
\RETURN $u$

\end{algorithmic}
\end{algorithm}

\begin{algorithm}
  \caption{Initialization of $\rho$}
  \label{alg:initialize_rho}
  \begin{algorithmic}[1]
  \REQUIRE HMM parameters $(\mu,C)$, obs.~sequence $z
  =[z_1,z_2,\hdots,z_T]$
  \ENSURE Measure $\rho = \{\mu,\rho_1,\hdots,\rho_{T-1},\rho_T\}$
  \STATE $\rho_0 \gets \mu$
  
  \FOR{$t = 1$ to $T$}
  \STATE $\rho_t \gets C(:,z_t)$
  \STATE $\rho_t \gets \text{normalize\_distribution}(\rho_t)$
  \ENDFOR
  
  \RETURN $\rho$
  
  \end{algorithmic}
  \end{algorithm}

\begin{algorithm}
\caption{Projection and normalization of measure}
\label{alg:normalize_rho}
\begin{algorithmic}[1]
\REQUIRE Signed measure $\sigma$
\ENSURE Probability measure $\rho$

\STATE $\rho \gets \max (\sigma,0)$ \hfill \% clip negative values to 0
\STATE $\rho \gets \rho / \text{sum}(\rho) $\hfill \% normalize
  
\RETURN $\rho$

\end{algorithmic}
\end{algorithm}

\begin{algorithm}
  \caption{Nonlinear prediction $p$}
  \label{alg:predict_p}
  \begin{algorithmic}[1]
  \REQUIRE HMM parameters $C$, measure $\rho =
  \{\rho_0,\rho_1,\hdots,\rho_{T-1},\rho_T\}\in\Re^{d\times (T+1)}$
  \ENSURE Measure $p = \{p_1,\hdots,p_{T}\}\in\Re^{m\times T}$
  
  \FOR{$t = 1$ to $T$}
  \STATE $p_{t} \gets \rho_{t} \cdot C$
  \ENDFOR
  
  \RETURN $p$
  
  \end{algorithmic}
\end{algorithm}

  \begin{algorithm}
    \caption{Sampling a random matrix}
    \label{alg:sparse_stochastic}
    \begin{algorithmic}[1]
    \REQUIRE row dim $d$, column dim $m$, temperature $\tau$
    \ENSURE Random matrix $M\in\Re^{d\times m}$
    \FOR{$i = 1$ to $d$}
    \STATE $M(i,:) \gets \text{softmax}(\text{randn}(m)/\tau)$
    \ENDFOR
    \RETURN $M$
    
    \end{algorithmic}
    \end{algorithm}

    \newpage
    
\section{Additional operations in a transformer}
\label{sec:xfer}

  This section is based on~\citep[Ch.,~10]{jm3}.  Fix $t$. The output $y_t$ from concatenating the output of multiple heads is subject to the following operations:
  \begin{enumerate}
  \item Residual connection which means
    \begin{equation*}
 y_t\mapsto y_t + e_t.
    \end{equation*}
  \item Layer normalization which means
    \begin{equation*}
  y_t\mapsto \text{diag}(\gamma) \frac{y_t-\text{mean}(y_t)}{\text{std}(y_t)} + \beta.
  \end{equation*}
  where $\gamma\in\Re^d$ and $\beta\in\Re^d$ are learnable parameters (referred to as gain and offset).
\item Feedforward neural network
  \begin{equation*}
  y_t\mapsto \text{FFN}(y_t).
  \end{equation*}
\end{enumerate}

\subsection{Summary of operations in a single layer}

 Fix time $t$.  The following operations define a single layer in a transformer:
 \begin{align*}
   y_t & = \text{MultiHeadAttention}(e_t,[e_1,e_2,\hdots,e_{t-1}])\\
   y_t & = y_t + e_t \\
   y_t &= \text{LayerNorm}(y_t) \\
   y_t & = y_t + \text{FFN}(y_t) \\
   y_t & = \text{LayerNorm}(y_t)
 \end{align*}
 The learnable parameters in the MultiHeadAttention are the matrices $W_V,W_K,W_Q,W_O$.  For each of the two LayerNorm operations, the learnable parameters are the gains $\gamma$ and the offset $\beta$.  Additionally, the weights of the FFN are also learned.

\section{Numerical Experiments with Transformer}
\label{sec:app_transformer}

\begin{table}[t]
  \centering
  \begin{tabularx}{0.55\textwidth}{l l l}
    \toprule
    \textbf{Parameter} & \textbf{Transformer} & \textbf{Dual filter} \\
    \midrule
    Context length & $T=256$ & $256$ \\
    Vocabulary size & $m=65$ & $65$ \\
    State (embedding) dim & $d=384$  & $384$ \\
    Number of layers & $L=6$ & 6 \\
    Number of heads & $n_{\text{head}}=6$ & N/A \\
    FFN dimension & 1536 & N/A\\
    Output temperature & 1 & N/A\\
    \bottomrule
  \end{tabularx}
    \hfill
  \begin{tabularx}{0.42\textwidth}{l l}
    \toprule
    \textbf{Training Parameter} & \textbf{nanoGPT} \\
    \midrule
    Learning rate & $10^{-3}$ \\
    Minimum learning rate & $10^{-4}$ \\
    Learning rate decay & cosine \\
    Learning rate decay steps & 2000 \\
    Optimizer & AdamW \\
    Batch size & 10 \\
    Dropout & 0.1 \\
    \bottomrule
  \end{tabularx}
  \caption{Summary of model parameters for the numerical experiments. The left table lists the architectural parameters, while the right table lists the training parameters specific to the nanoGPT transformer.}
  \label{tab:model_parameters}
  \vspace*{-25pt}
\end{table}

The numerical experiments described in \Sec{sec:numerics} were
repeated using a nanoGPT transformer~\citep{karpathy2024nanogpt}.  The
observations are generated from an HMM.  The HMM parameters are the same
as described in \Sec{sec:numerics}, specifically, $d = 384$, $m = 65$, and $T =
256$, and other parameters are as follows:
\begin{itemize}
\item  The prior $\mu$ is set to the uniform probability vector: $
  \mu(x) = \frac{1}{d}$ for $x\in\bS$.
\item The transition matrix $A=A^\text{(circ)}$.
\item The emission matrix $C$ is randomly sampled. See
  Algorithm~\ref{alg:sparse_stochastic} for details.
\end{itemize}
For this model, the numerical results using the dual filter are given 
in~\Fig{fig:one_shot_algorithm_results}
and~\Fig{fig:iterative_algorithm_results}. 

As a first step, the nanoGPT is trained using the observation data
sampled from the $\text{HMM}(\mu,A,C)$.  The nanoGPT hyper-parameters were carefully tuned to
ensure optimal performance (see Table~\ref{tab:model_parameters}).  Note that the dual filter theory is
entirely model-based.  So, there is no training step to obtain the
numerical results reported in \Sec{sec:numerics}.

As a second step, the trained nanoGPT model is used for inference and
its prediction performance numerically evaluated.  For this purpose,
the observations $z=\{z_1,z_2,\hdots,z_T\}$ are generated from the
same $\text{HMM}(\mu,A,C)$ used for training.  The inference results were
mixed, depending upon the randomly sampled $C$.  There are two cases as follows:
\begin{itemize}
\item \textbf{Favorable case}. See left panel of
  \Fig{fig:transformer_results}.  Here, nanoGPT produces
  predictions $p_t$ that are good approximations of the ground truth
  computed from the model-based 
  nonlinear filter.  Quantitatively, the error
  $\varepsilon_t^{(\ell)}$ starts near $1$ at the first layer and
  progressively decays with each subsequent layer, reaching values in
  the range $10^{-1}$ to $10^{-4}$ at the final layer ($\ell = L$). 
The top-ten conditional probabilities closely match the ground
truth for all time points, indicating successful
learning and inference. 
\item \textbf{Unfavorable case}. See right panel of
  \Fig{fig:transformer_results}.  Here, the nanoGPT predictions are
  not accurate.  
The error $\varepsilon_t^{(\ell)}$ remains high throughout all layers and does not exhibit the progressive decay seen in the favorable case.
The top-ten conditional probabilities deviate substantially from the
ground truth, reflecting a failure to learn an accurate predictive
distribution. 
\end{itemize}
Notably, the results described in~\Fig{fig:one_shot_algorithm_results}
and~\Fig{fig:iterative_algorithm_results} are for the unfavorable
case.

The numerical results suggest that while the transformer can learn to
approximate the nonlinear filter in some cases, its performance is
sensitive to the choice of the emission matrix $C$. An empirical
observation is that in the favorable case, the conditional probability
$p_t$ is supported on a relatively small subset of the vocabulary
$\bO$ --- suggesting that the transformer performs well when the
prediction task is, in some sense, low-complexity. Without a suitable
theoretical framework, it is difficult to make this precise.




\begin{figure}[t]
  \centering
  \begin{minipage}{0.48\textwidth}
    \centering
    \includegraphics[width=\textwidth]{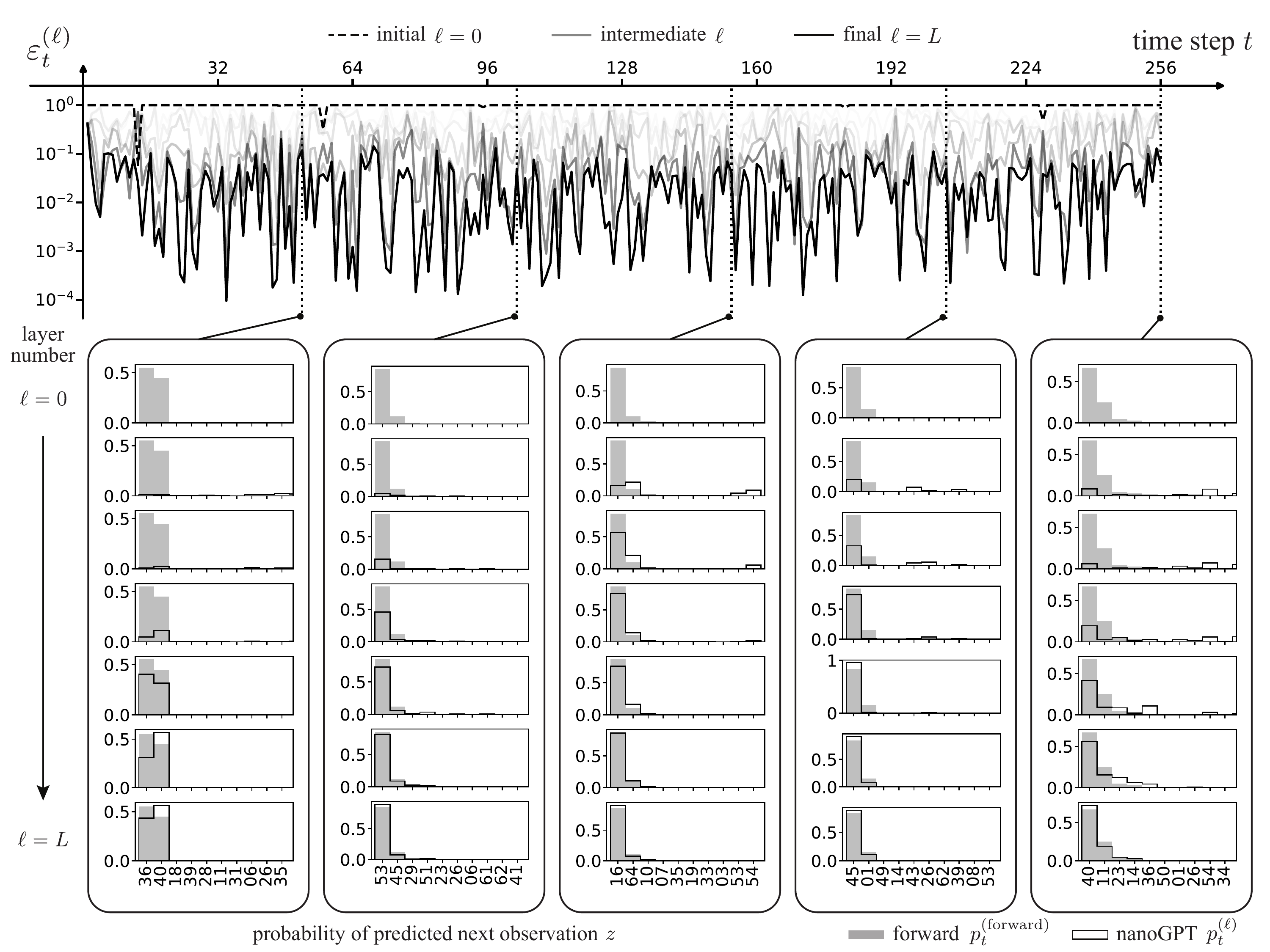}
  \end{minipage}
  \hfill
  \begin{minipage}{0.48\textwidth}
    \centering
    \includegraphics[width=\textwidth]{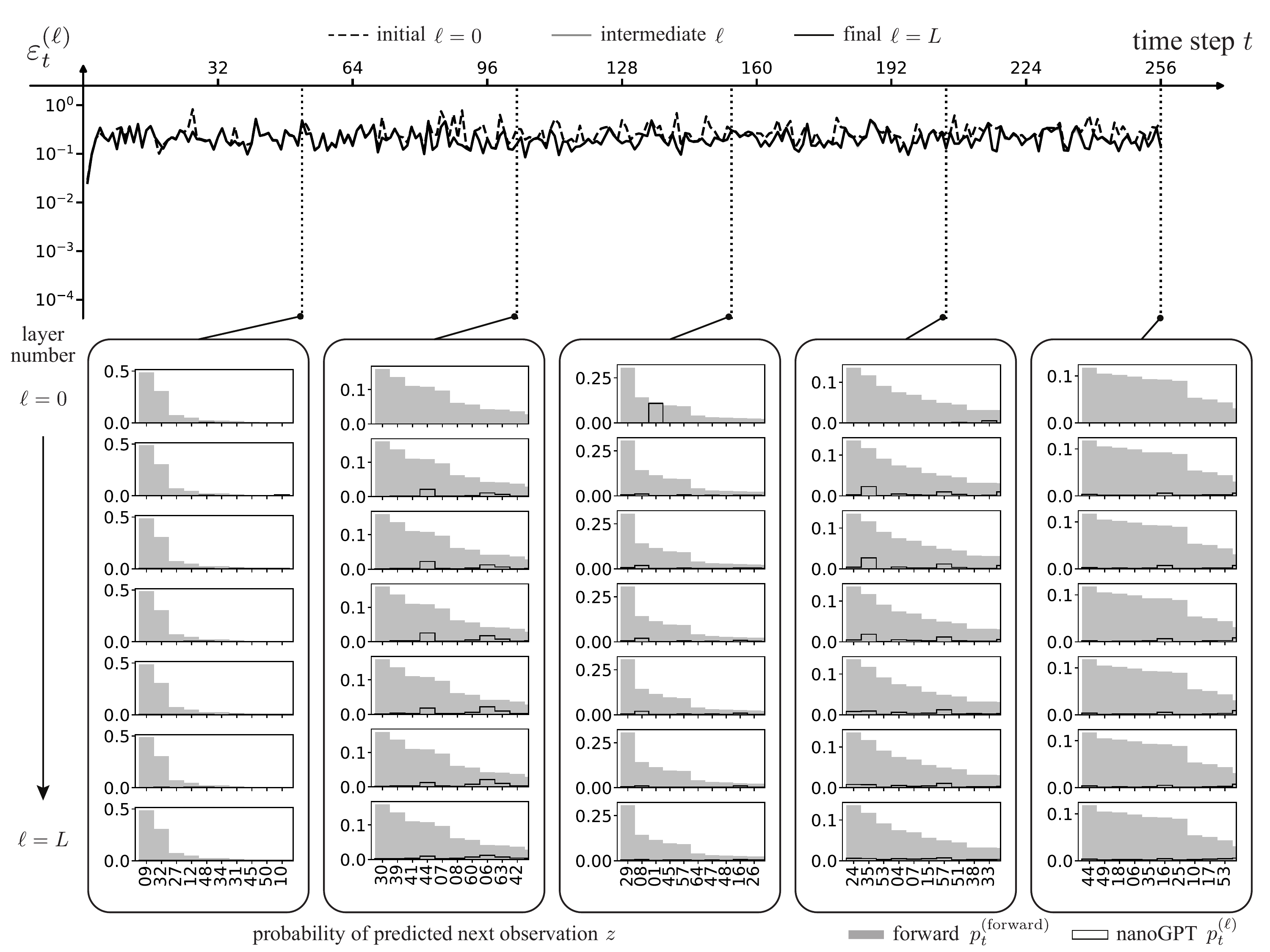}
  \end{minipage}
  \caption{
Comparison of transformer predictions against the ground truth from
the nonlinear filter, for two choices of the randomly sampled emission
matrix $C$. 
    \textbf{(Left)} Favorable case: the error decays from near $1$ to the range $10^{-1}$--$10^{-4}$ by the final layer, and the predicted conditional probability closely matches the ground truth.
    \textbf{(Right)} Unfavorable case (same HMM instance as
    \Fig{fig:iterative_algorithm_results}): the error remains high
    across all layers and the predicted conditional probability does
    not converge to the ground truth. Each panel shows: (Top) time traces of the per-step error $\{\varepsilon_t^{(\ell)} : 1 \leq t \leq T\}$ for layers $\ell=1,2,\hdots,L$ (with $L=6$); the dashed line is the error immediately after embedding ($\ell = 0$), and solid lines show subsequent layers with darker shades used as $\ell$ increases.
    (Bottom) top-ten conditional probabilities $\{p_t^{(\ell)}(z_i)\}$ at five representative time points, sorted in decreasing order; the gray shading is the ground truth and the solid line is the transformer output.
  }
  \vspace*{-25pt}
  \label{fig:transformer_results}
\end{figure}

\bibliography{bibfiles/_master_bib_jin,bibfiles/jin_papers,bibfiles/extrabib,bibfiles/estimator_controller,bibfiles/transformer,bibfiles/transformer_math}

\end{document}

%% file: _definitions.tex
\def\Re{\mathbb{R}}

\def\Sec#1{Sec.~\ref{#1}}
\def\Fig#1{Fig.~\ref{#1}}

\def\notes#1{\marginpar{\tiny #1}\typeout{Notes!
Notes!
Notes!
}}
\renewcommand{\notes}[1]{\typeout{notes!}}
\def\FRAC#1#2#3{\genfrac{}{}{}{#1}{#2}{#3}}
\def\half{{\mathchoice{\FRAC{1}{1}{2}}%
{\FRAC{2}{1}{2}}%
{\FRAC{3}{1}{2}}%
{\FRAC{4}{1}{2}}}}

\def\Re{\field{R}}

\def\Sec#1{Sec.~\ref{#1}}

\def\transpose{{\hbox{\rm\tiny T}}}

\def\clP{{\cal P}}
\def\clZ{{\cal Z}}

\def\Sec#1{Sec~\ref{#1}}

\def\E{{\sf E}}

\def\Fig#1{Fig.~\ref{#1}}
\def\Sec#1{Sec.~\ref{#1}}

\def\clZ{{\cal Z}}

\def\beq{\begin{eqnarray}} 
\def\bc{\begin{center}} 
\def\be{\begin{enumerate}}
\def\bi{\begin{itemize}} 
\def\bs{\begin{small}}
\def\bS{\begin{slide}}
\def\ec{\end{center}} 
\def\ee{\end{enumerate}}
\def\ei{\end{itemize}}
\def\es{\end{small}}
\def\eS{\end{slide}}
\def\eeq{\end{eqnarray}}

\newcommand{\newP}[1]{\medskip\noindent{\bf #1:}}

\def\Re{\mathbb{R}}
\def\E{{\sf E}}

\def\clY{{\cal Y}}

\def\Sec#1{Sec.~\ref{#1}}
\def\Thm#1{Thm.~\ref{#1}}
\def\Prop#1{Prop.~\ref{#1}}

\def\clD{{\cal D}}

\def\clP{{\cal P}}
\def\clZ{{\cal Z}}

\renewcommand{\Re}{\mathbb{R}}

\def\FRAC#1#2#3{\genfrac{}{}{}{#1}{#2}{#3}}

\newcommand{\var}{\text{var}}



%% file: jin_thesis_macros.tex
\def\clD{{\cal D}}

\def\clF{{\cal F}}
\def\clG{{\cal G}}

\def\clM{{\cal M}}
\def\clN{{\cal N}}

\def\clP{{\cal P}}

\def\clU{{\cal U}}
\def\clV{{\cal V}}

\def\clY{{\cal Y}}
\def\clZ{{\cal Z}}
\def\E{{\sf E}}
\def\bS{\mathbb{S}}
\def\bO{\mathbb{O}}
\def\sJ{{\sf J}}

\def\ones{{\sf 1}}
\def\sP{{\sf P}}

\def\tp{{\hbox{\rm\tiny T}}}

\def\dvar{\operatorname{var}}

\def\opt{{\text{\rm (opt)}}}

\def\xfer{{\text{\rm (xfer)}}}
\def\dfltr{{\text{\rm (dfltr)}}}


%% file: figure/hmm.tex








\begin{tikzpicture}[
    node distance=1.2cm,
    latent/.style={draw, circle, minimum size=1cm, inner sep=0pt},
    obs/.style={draw, circle, minimum size=1cm, inner sep=0pt, fill=gray!20},
    arrow/.style={-stealth, thick},
    dots/.style={inner sep=3pt, minimum size=0.2cm}
]

    \node (start_x) {};
    \node (latent_label) [left=0.3cm of start_x] {state};
    \node (observation_label) [above=0.7 cm of latent_label] {observation (token)};

    \node (x0) at ($(start_x) + (0.5cm, 0)$) [latent] {$X_0$};
    \node (x1) [latent, right=of x0] {$X_1$};
    \node (dots) [dots, right=of x1] {$\cdots$};
    \node (xT_1) [latent, right=of dots] {$X_{T-1}$};
    \node (xT) [latent, right=of xT_1] {$X_T$};

    \node (z1) [obs, above=0.5cm of x1] {$Z_1$};
    \node (z2) [obs, right=of z1] {$Z_2$};
    \node (zd) [right=of z2] {$\cdots$};
    \node (zT) [obs, right=of zd] {$Z_T$};
    \node (zT_plus_1) [obs, right=of zT] {$Z_{T+1}$};

    \draw [arrow] (x0) -- (x1);
    \draw [arrow] (x1) -- (dots);
    \draw [arrow] (dots) -- (xT_1);
    \draw [arrow] (xT_1) -- (xT);

    \draw [arrow] (x0) -- (z1);
    \draw [arrow] (x1) -- (z2);
    \draw [arrow] (xT_1) -- (zT);
    \draw [arrow] (xT) -- (zT_plus_1);

    \draw [arrow] ($(x0) + (-0.5cm, -1.0cm)$) -- ($(xT) + (0.5cm, -1.0cm)$) node[right] {time};

    \draw ($(x0) + (0, -1.0cm)$) -- ($(x0) + (0, -1.2cm)$) node[below] {$0$};
    \draw ($(x1) + (0, -1.0cm)$) -- ($(x1) + (0, -1.2cm)$) node[below] {$1$};
    \draw ($(xT_1) + (0, -1.0cm)$) -- ($(xT_1) + (0, -1.2cm)$) node[below] {$T-1$};
    \draw ($(xT) + (0, -1.0cm)$) -- ($(xT) + (0, -1.2cm)$) node[below] {$T$};

\end{tikzpicture}

%% file: figure/transformer_layer.tex
\begin{tikzpicture}[
  transformer/.style={draw, thick, minimum width=5cm, minimum height=1.5cm},
  inputarrow/.style={thick},
  outputarrow/.style={thick}
  ]
\begin{scope}[shift={(-7, 0)}]
  \node[transformer] (transformer) at (0,0) {$\clN^\dfltr$};

  \foreach \i/\x in {1/-2, 2/-1, T/2} {
    \node at (\x, -2) (z\i) {$\rho_\i$};
    \draw[inputarrow] (z\i) -- (\x, -0.75);
  }

  \node at (0.5, -2.1) (dots) {$\ldots$};

  \foreach \i/\x in {1/-2, 2/-1, T/2} {
    \node at (\x, 2) (z+\i) {$\rho_\i^{+}$};
    \draw[outputarrow] (z+\i) -- (\x, 0.75);
  }

  \node at (0.5, 1.9) (dots) {$\ldots$};

  \node at (-4, -2) {(input)};
  \node at (-4, 2) {(output)};

  \node at (0, -3) {Dual filter layer};
\end{scope}

\begin{scope}
  \node[transformer] (transformer) at (0,0) {$\clN^\xfer$};

  \foreach \i/\x in {1/-2, 2/-1, T/2} {
    \node at (\x, -2) (z\i) {$\sigma_\i$};
    \draw[inputarrow] (z\i) -- (\x, -0.75);
  }

  \node at (0.5, -2.1) (dots) {$\ldots$};

  \foreach \i/\x in {1/-2, 2/-1, T/2} {
    \node at (\x, 2) (z+\i) {$\sigma_\i^{+}$};
    \draw[outputarrow] (z+\i) -- (\x, 0.75);
  }

  \node at (0.5, 1.9) (dots) {$\ldots$};

  \node at (0, -3) {Transformer self-attention layer};

\end{scope}

\end{tikzpicture}